\documentclass{article}
\usepackage{ijcai24}

\usepackage{times}
\usepackage{soul}
\usepackage{url}
\usepackage[hidelinks]{hyperref}
\usepackage[utf8]{inputenc}
\usepackage[small]{caption}
\usepackage{graphicx}
\usepackage{amsmath}
\usepackage{amsthm}
\usepackage{booktabs}
\usepackage{algorithm}
\usepackage[noend]{algpseudocode}
\usepackage[switch]{lineno}
\usepackage{nicefrac}

\usepackage{xcolor}
\usepackage{xspace}
\usepackage{multirow}
\usepackage{amssymb}
\usepackage{mathtools}
\usepackage{mdframed}
\usepackage{bm} 

\newtheorem{example}{Example}
\newtheorem{theorem}{Theorem}
\newtheorem{definition}{Definition}
\newtheorem{claim}{Claim}
\newtheorem{corollary}{Corollary}
\newtheorem{lemma}{Lemma}

\DeclareMathOperator*{\argmax}{arg\,max}

\newcommand{\citet}[1]{\citeauthor{#1}~\shortcite{#1}}
\newcommand{\ourLB}{PF-MAA$^*$\xspace}
\newcommand{\ourUB}{TR-MAA$^*$\xspace}
\newcommand{\rsmaa}{RS-MAA$^*$\xspace}
\newcommand{\maastar}{MAA$^*$\xspace}
\newcommand{\gmaastar}{GMAA$^*$\xspace}
\newcommand{\astar}{A$^*$\xspace}
\newcommand{\gmaaice}{GMAA$^*$-ICE\xspace}

\newcommand{\benchmark}[1]{\textsc{#1}\xspace}
\newcommand{\dectiger}{\benchmark{DecTiger}}
\newcommand{\boxpush}{\benchmark{BoxPushing}}
\newcommand{\ff}{\benchmark{FireFighting}}

\newcommand{\recycling}{\benchmark{Recycling}}
\newcommand{\hotel}{\benchmark{Hotel}}
\newcommand{\broadcast}{\benchmark{Broadcast}}
\newcommand{\mars}{\benchmark{Mars}}
\newcommand{\grid}{\benchmark{Grid}}
\newcommand{\gridthree}{\benchmark{Grid3x3}}

\newcommand{\ppolicies}{\Phi}
\newcommand{\extendedby}{<_E}
\newcommand{\jointppol}{\boldsymbol{\varphi}}
\newcommand{\ppol}{\varphi}

\newcommand{\ext}[1]{E(#1)}
\newcommand{\genpolicies}{\Pi_\mathsf{gen}}
\newcommand{\extgen}[1]{E_\mathsf{gen}(#1)}
\newcommand{\obstracealt}{\widetilde{\obstrace}}
\newcommand{\jointobstracealt}{\widetilde{\jointobstrace}}
\newcommand{\stage}[1]{\sigma(#1)}

  \newcommand{\obsaltaltalt}{\bar{\obs}}
  \usepackage{tikz}
\tikzset{every picture/.style={line width=0.75pt}} 

\newcommand{\midd}{\,\middle|\,}

\newcommand{\multitwo}[1]{\multirow{2}{*}{\parbox[c]{1.2cm}{\centering #1}}}

\algrenewcommand\algorithmicindent{1.0em}%

\graphicspath{{figures/}}

\title{Approximate Dec-POMDP Solving Using Multi-Agent A$^*$}

\author{
Wietze Koops$^1$ \and
Sebastian Junges$^1$ \And
Nils Jansen$^{2,1}$
\affiliations
$^1$Radboud University, Nijmegen, The Netherlands \\
$^2$Ruhr-University Bochum, Germany
\emails
\{wietze.koops, sebastian.junges\}@ru.nl, n.jansen@rub.de
}

\begin{document}

\maketitle

\begin{abstract}
 We present an \astar-based algorithm to compute policies for finite-horizon Dec-POMDPs.
 Our goal is to sacrifice optimality in favor of scalability for larger horizons.  
The main ingredients of our approach are (1)~using clustered sliding window memory, (2)~pruning the \astar search tree, and (3)~using novel \astar heuristics.  
Our experiments show competitive performance to the state-of-the-art. 
Moreover, for multiple benchmarks, we achieve superior performance. 
 In addition, we provide an \astar algorithm that finds upper bounds for the optimum, tailored towards problems with long horizons. The main ingredient is a new heuristic that periodically reveals the state, thereby limiting the number of reachable beliefs. Our experiments demonstrate the efficacy and scalability of the approach.

\end{abstract}

\section{Introduction}

Decentralized partially observable Markov decision processes (Dec-POMDPs) formalize multi-agent decision-making under stochastic dynamics and partial observability. They are, for instance,  suitable to model bandwidth allocation~\cite{DBLP:conf/ssci/HemmatiYS15} and maintenance problems~\cite{bhustali2023assessing}. The decision problem underlying solving Dec-POMDPs exactly or $\epsilon$-optimally is  NEXP-hard~\cite{bernstein2002complexity,DBLP:conf/atal/RabinovichGR03}. This paper explores a model-based approach for approximate solving of finite-horizon Dec-POMDPs.  Concretely, our approach finds policies that obtain a high value on a given Dec-POMDP and bounds on the value achieved by an optimal policy. The proximity between the achieved values and the bounds shows that the policies empirically perform very well.

\paragraph{Small-step \maastar.}
A prominent line of work for solving Dec-POMDPs for finite horizons builds upon multi-agent \astar (\maastar)~\cite{szer2005maa}. The crux of these algorithms is to search through the space of all joint policies in a search space where we incrementally fix the decisions of individual agents. To alleviate a doubly-exponential out-degree with growing horizons, \emph{small-step} \maastar \cite{DBLP:conf/ijcai/Koops0JS23} makes these decisions for every observation history sequentially, yielding very deep search trees that have a limited out-degree. 
However, so far, \maastar is used to compute exact solutions by iteratively refining the upper bound provided by an admissible heuristic. In this paper, we show that \maastar provides a competitive foundation for algorithms that find good but not necessarily optimal policies. Likewise, \maastar is also a solid foundation for computing non-trivial (and potentially tight) upper bounds. 
The result is an algorithm that scales to much higher horizons than (exact) \maastar-based algorithms.
Below, we briefly give our perspective on both lower and upper bounds before outlining the main \textit{technical ingredients} (TIs). 

\paragraph{Lower bounds by finding policies.} To find good policies fast, we limit the space of policies by considering policies that are independent of old observations (TI1) and we limit the number of partial policies that we may explore at any level of the \astar search tree (TI2).
Since we already limit the number of policies expanded, it is not essential to use a tight heuristic for this. We show that \maastar can also find good policies fast, using heuristics that are generally not tight (TI3).

\paragraph{Proving upper bounds.} When using \maastar with any admissible heuristic, the highest heuristic value is an upper bound for the optimal value. However, the only  heuristic that scales to the horizons for which we find lower bounds is $Q_{\mathsf{MDP}}$, which treats the Dec-POMDP as a fully-observable, centralized MDP. 
Other, tighter heuristics cannot handle the horizons we are interested in. 
In this paper, we introduce novel heuristics (TI4) which are less tight but more scalable than the heuristics previously used for solving Dec-POMDPs.

\paragraph{Technical ingredient 1: Clustering with sliding-window memory.} 
A major challenge in finding good policies for large horizons is the fact that the optimal policy may be exponentially large.
Lossless incremental \emph{clustering}~\cite{oliehoek2009lossless} helps to determine that an optimal policy may, w.l.o.g., take the same decision based on different observation histories. However, even with clustering, computing a policy generally requires determining an action for exponentially many different observation histories. In this work, we consider clustering for a specific subclass of policies. In particular, we consider sliding window policies that only depend on the most recent $k$ observations, which empirically are often the most relevant observations. 
We develop a lossless clustering for sliding window memory, which clusters these windows with no additional loss of optimal policy value, compared to using sliding window memory.

\paragraph{Technical ingredient 2: Pruning the queue.} \astar-based algorithms explore partial policies in order of the heuristic value. Especially if the heuristic is not tight, this may lead to a breadth-first-like exploration of policies, effectively preventing exploring policies that are defined for longer horizons. We therefore prevent considering partial policies by pruning nodes in the search tree. We use a hard cap for every stage of the search tree while ensuring that we do eventually expand a complete (i.e., non-partial) policy.

\paragraph{Technical ingredient 3: Loose heuristics.}
In any \astar-variant, policies are analysed in the order suggested by the heuristic. A good heuristic is thus one which leads early on to a policy with a high value. While tight heuristics do ensure this property, also loose heuristics can have a similar property. 
Empirically, it is often near-optimal to greedily optimize for the next few steps. Concretely, our heuristic considers only the next few time steps using a Dec-POMDP. The potential reward after these few steps is estimated using an MDP or even by assuming that the maximal reward is constantly achieved.

\paragraph{Technical ingredient 4: Scalable and tight heuristics for upper bounds.} 
To prove upper bounds on the optimal value, we need tight admissible heuristics.
Common heuristics relax the restrictions of a decentralized, partially observable setting and assume the problem to be centralized and/or fully observable. 
However, full state-observability yields a heuristic that is not sufficiently tight, and relaxing the setting to a (centralized) POMDP does not avoid expensive computations as it requires considering exponentially many beliefs in the horizon. 
Inspired by the idea of only sharing information after some delay, we present an admissible heuristic that \emph{periodically} reveals state information. 
This yields a trade-off between the horizons over which one must reason about partial information and the tightness of the heuristic. The heuristic can be computed on all benchmarks we selected.

\paragraph{Contributions.} To summarize, our technical advancements on clustering, heuristics, and pruning outlined above together yield a pair of algorithms that  find good policies as well as upper bounds for Dec-POMDPs for horizons more than an order of magnitude larger than for which exact Dec-POMDP solving is possible. In particular, for the \boxpush benchmark with horizons up to $100$, we find policies with values that are only 1\% smaller than our upper bounds.\footnote{Source code is available at \url{https://zenodo.org/records/11160648}}

\subsubsection*{Related Work}

This work builds on a series of works on \astar algorithms for solving Dec-POMDPs, culminating in the exact algorithms \gmaaice \cite{oliehoek2013incremental} and \rsmaa \cite{DBLP:conf/ijcai/Koops0JS23}. Using \astar algorithms for approximate solving of Dec-POMDPs has also been proposed previously. In particular, \citet{oliehoek2008optimal} propose $k$-\gmaastar, which in each step only adds the $k$ best children of each node as computed using a Bayesian game. For better scalability, \citet{DBLP:conf/atal/Emery-MontemerloGST04} combine this (for $k=1$) with only approximately solving the Bayesian games and lossy clustering. Instead of only adding the $k$ best children, our algorithm adds all children and prunes them when necessary (TI2). Therefore, our algorithm is able to use its resources to search at points where it is less clear what is the best action to choose. Finally, \citet{szer2005maa} mention a heuristic where the overestimate for future reward is weighted with some factor $w < 1$. This is not an admissible heuristic, but it results in \maastar finding a (possibly suboptimal) policy faster.

Related work in the general \astar literature includes \citet{cazenave2010partial}, which proposed an idea similar to small-step \maastar in a general setting. \citet{DBLP:conf/ecai/Russell92} studies Simplified Memory-Bounded \astar (SM\astar), which bounds the memory required by \astar until the first solution is found by pruning policies with a low heuristic from the priority queue.

Early work on approximately solving Dec-POMDPs includes
JESP \cite{nair2003taming}, which computes a Nash equilibrium, and DICEPS \cite{DBLP:journals/informaticaSI/OliehoekKV08} which uses the cross-entropy method. Another line of work is on algorithms that use dynamic programming~(DP) \cite{hansen2004dynamic,DBLP:conf/uai/SeukenZ07,DBLP:conf/atal/CarlinZ08,DBLP:conf/flairs/KumarZ09,DBLP:conf/atal/DibangoyeMC09,amato2009incremental}. Although these DP algorithms are bottom-up rather than top-down, they also use pruning to limit the number of policies (per time step). In addition, they use techniques reminiscent of clustering, to avoid spending too much time on optimizing for unlikely observations.

The genetic algorithm GA-FSC \cite{DBLP:journals/aamas/EkerA13} searches through finite state controllers and finds the best known policies on several benchmarks. The state-of-the-art $\epsilon$-optimal algorithm FB-HSVI transforms the Dec-POMDP into a continuous-state MDP  \cite{DBLP:journals/jair/DibangoyeABC16} and
solves it using an adaption of heuristic search value iteration.

There is also significant work on model-free reinforcement learning (RL) algorithms \cite{DBLP:journals/ijon/KraemerB16,DBLP:conf/pkdd/BonoDMP018,DBLP:conf/jfpda/DibangoyeB18,DBLP:conf/cdc/MaoZMB20}. The model-based RL algorithm Team-Imitate-Synchronize \cite{DBLP:conf/pkdd/AbdooBSS22} learns a centralized team policy, which is imitated by a decentralized policy and improved using synchronization.

\section{Problem Statement}

\newcommand{\agents}{\mathcal{D}}
\newcommand{\states}{\mathcal{S}}
\newcommand{\initbelief}{b}
\newcommand{\actions}{\mathcal{A}}
\newcommand{\jointactions}{\bm{\actions}}
\newcommand{\act}{a}
\newcommand{\jointact}{\mathbf{\act}}
\newcommand{\transitions}{T}
\newcommand{\rewards}{R}
\newcommand{\observations}{\mathcal{O}}
\newcommand{\jointobservations}{\bm{\observations}}
\newcommand{\obs}{o}
\newcommand{\jointobs}{\mathbf{o}}
\newcommand{\obsfun}{O}
\newcommand{\horizon}{h}
\renewcommand{\time}{t}
\newcommand{\policy}{\pi}
\newcommand{\jointpolicy}{{\bm{\policy}}}
\newcommand{\distr}[1]{\Delta(#1)}
\newcommand{\tuple}[1]{\langle #1 \rangle}
\newcommand{\jointpolicies}{\Pi}
\newcommand{\obstrace}{\tau}
\newcommand{\jointobstrace}{\bm{\obstrace}}

 \newcommand{\clusterpol}{\policy^C}
 \newcommand{\jointclusterpol}{\jointpolicy^C} 
 \newcommand{\clusterpolicies}{\jointpolicies^C}
 \newcommand{\clusterfun}{C}
 \newcommand{\jointcluster}{\mathbf{c}}
  \newcommand{\forget}{F}
 \newcommand{\cluster}{c}
 \newcommand{\clusteralt}{\tilde{c}}
 \newcommand{\obsalt}{\tilde{\obs}}
  \newcommand{\clusteraltalt}{\hat{c}}
 \newcommand{\obsaltalt}{\hat{\obs}}
 \newcommand{\windowsize}{k}

We briefly recap Dec-POMDPs \cite{DBLP:series/sbis/OliehoekA16} following the notation of \citet{DBLP:conf/ijcai/Koops0JS23}. 
In particular, $\distr{X}$ denotes the set of distributions over a finite set $X$.

\begin{definition}[Dec-POMDP]
A \emph{Dec-POMDP} is a tuple $\tuple{\agents, \states, \jointactions, \jointobservations, \initbelief, \transitions, \rewards, \obsfun}$ with
    a set $\agents = \{ 1, \dots, n \}$ of $n$ \emph{agents},
    a finite set $\states$ of \emph{states},
    a set $\jointactions = \bigtimes_{i \in \agents} \actions_i$ of \emph{joint actions}, and
    a  set $\jointobservations = \bigtimes_{i \in \agents}\observations_i$ of \emph{joint observations}, where $\actions_i$ and $\observations_i$ are finite sets of \emph{local actions} and \emph{local observations} of agent $i$.
    The \emph{transition function} $\transitions \colon \states \times \jointactions \rightarrow \distr{\states}$ defines the \emph{transition probability} $\Pr(s' \mid s, \jointact)$,
    $\initbelief \in \distr{\states}$ is the \emph{initial belief}, $\rewards \colon \states \times \jointactions \rightarrow \mathbb{R}$ is the \emph{reward} function, and
    the \emph{observation function} $\obsfun \colon \jointactions \times \states \rightarrow \distr{\jointobservations}$ defines the \emph{observation probability} $\Pr(\jointobs \mid \jointact, s')$.
\end{definition}
 \noindent A Dec-POMDP describes a system whose state changes stochastically at every stage. The initial state is $s^0$ with probability $\initbelief(s^0)$. At each stage $t$, each agent $i$ takes a local action $\act_i^t$, resulting in a joint action $\jointact^t = \tuple{\act_1^t, \dots, \act_n^t}$ and a reward $r^t  = \rewards(s^t,\jointact^t)$. The next state is $s^{t+1}$ with probability $\Pr(s^{t+1} \mid s^t, \jointact^t)$. Finally, a joint observation $\jointobs^{t+1}$ is drawn with probability $\Pr(\jointobs^{t+1} \mid \jointact^t, s^{t+1})$, and each agent $i$ receives their local observation $\obs_i^{t+1}$.

 We write  $\obs_i^{[v, t]} = \obs_i^{v}\ldots\obs_i^t$ for the  local observations of agent $i$ between stage $v$ and $t$. If the next local observation is $o \in \observations_i$, then $\obs_i^{[v, t]} \cdot o$ denotes the local observations between stage $v$ and $t+1$. We write $\obstrace_i = \obs_i^{[1, t]}$ for the \emph{local observation history} (LOH) of agent $i$, and $\jointobstrace = \jointobs^1\dots \jointobs^t$ for the \emph{joint observation history}.
 We denote the set of all LOHs of agent $i$ of length exactly $\ell$ and at most $\ell$ by $\observations_i^{\ell}$ and $\observations_i^{\leq \ell}$, respectively.

 Agents can choose their action based on all their past observations, i.e.\ based on their LOH. 
 A \emph{local policy} for agent $i$ maps LOHs for that agent to a local action, formally: $\policy_i \colon \observations_i^{\leq h-1} \rightarrow \actions_i$.
 A \emph{joint policy} is a tuple of local policies $\jointpolicy = \tuple{\policy_1, \cdots, \policy_n}$, and $\jointpolicies$ denotes the set of all joint policies.
 We will often refer to a joint policy simply as a \textit{policy}. 

Given a policy $\jointpolicy$, we define the value of executing this policy in the initial belief $b$ over a horizon $h$ as:
\[
 V_\jointpolicy(\initbelief, \horizon) = \mathbb{E}_{\jointpolicy}\left[\sum_{t=0}^{h-1} R(s^t, \jointact^t) \,\middle|\, s^0 \sim b \right],
\]
 where $s^t$ and $\jointact^t$ are the state and joint action at stage $t$. 
Finally, in all probabilities, we implicitly also condition on the past policy and the initial belief.

The goal of the agents is to maximize the expected reward. 
An optimal policy is a policy $\jointpolicy^* \in { \argmax}_{\jointpolicy' \in \jointpolicies} V_{\jointpolicy'}(b, \horizon)$. The aim of this paper is to find a policy $\jointpolicy$ with a high value, as well as upper bounds for the optimal value. 

 \begin{mdframed}
 \textbf{Problem statement:} Given a Dec-POMDP and a horizon $\horizon$, find a policy $\jointpolicy$ and an upper bound $U(b, \horizon)$ s.t.
 \[ V_{\jointpolicy}(b, \horizon) \quad \leq\quad\max_{\jointpolicy' \in \jointpolicies} V_{\jointpolicy'}(b, \horizon) \quad\leq\quad U(b, \horizon), \]
 and $V_{\jointpolicy}(b, \horizon)$ and $U(b, \horizon)$ close to each other.
 \end{mdframed}

\paragraph{Sliding window memory.} We search the space of \emph{sliding $\windowsize$-window memory} policies that depend only on the last $\windowsize$ observations. Windows that end at stage $t$ start at stage $\ell_t = \max(t-\windowsize, 0)+1$. A policy $\pi_i$ has \emph{sliding window memory}, if  $\obs^{[\ell_t,t]}_i = \obsalt^{[\ell_t,t]}_i$ implies $\pi_i\big(\obs^{[1,t]}_i\big) = \pi_i\big(\obsalt^{[1,t]}_i\big)$.

 \section{Clustering}\label{sec:clustering}

Abstractly, our algorithm searches over policies that map LOHs to actions. To limit the search space, we adopt clustering~\cite{oliehoek2009lossless}, i.e., the idea that policies assign the same action to LOHs that belong to the same cluster of LOHs. Firstly, we introduce formally sliding $k$-window memory. Then we introduce \emph{clustered} sliding $k$-window memory, which aims to merge existing clusters further without inducing a loss in policy value.

\begin{definition}
    A \emph{clustering} is a partition $C_{i, t}$ of $\observations_i^t$ for each stage $0 \leq t \leq \horizon-1$ and each agent $i \in \agents$. 
\end{definition}

\begin{definition}\label{def:sliding}
    The clustering for sliding $\windowsize$-window memory consists of the partitions $C_{i,t}$,  $0 \leq t \leq \horizon{-}1$, where \[C_{i,t} = \left\{\left\{\obsalt^{[1,t]}_i \!\in\! \observations_i^t \midd\obsalt^{[\ell_t,t]}_i \!=\! \obs^{[\ell_t,t]}_i \right\} \midd \obs^{[\ell_t,t]}_i \in \observations_i^{\min\{t, k\}} \right\}.\] 
\end{definition}
 
We identify the equivalence class (cluster) of all LOHs that have suffix $\obs^{[\ell_t,t]}_i$ with exactly this suffix $\obs^{[\ell_t,t]}_i$.

\paragraph{Cluster policies.} We now introduce \emph{cluster policies}.
Write $C_i = \bigcup_{t=0}^{\horizon-1} C_{i,t}$ for the set of agent $i$'s clusters. 
Let $C \colon \bigcup_{i \in \agents}  \bigcup_{t=0}^{\horizon-1} \observations_i^t \rightarrow \bigcup_{i \in \agents} C_i$ be the map that assigns an LOH to its cluster. 
A \emph{local cluster policy} for agent $i$ is a map $\clusterpol_i \colon C_i \rightarrow \actions_i$. 
The corresponding local policy $\policy_i$ is defined by $\policy_i(\obstrace_i) = \clusterpol_i(\clusterfun(\obstrace_i))$ for each LOH $\obstrace_i$. 
We denote the set of all cluster policies  $\jointclusterpol = \tuple{\clusterpol_1, \ldots, \clusterpol_n}$ by $\clusterpolicies$.  
The value of a cluster policy $\jointclusterpol$ is the value of the corresponding policy $\jointpolicy = \tuple{\policy_1, \ldots, \policy_n}$, where each $\policy_i$ is the local policy corresponding to $\clusterpol_i$. 
Note that a best cluster policy corresponding to sliding $\windowsize$-window memory is not necessarily optimal.

\subsubsection{Clustered Sliding Window Memory} 
We extend the so-called lossless clustering \cite{oliehoek2009lossless} to sliding window memory. 
We require our clustering to be \emph{incremental}: if two LOHs are clustered, then their extensions by the same observation are also clustered together.

\begin{definition}
   Let $\equiv_{C_{i, t}}$ be the equivalence relation induced by the partition $C_{i, t}$. A clustering is \emph{incremental} at stage $t$ if 
    \begin{align*} \forall \obs_i^{[1,t]} \equiv_{C_{i,t}} \obsalt_i^{[1,t]} \forall \obs \in \observations_i:  \big( \obs_i^{[1,t]} \cdot \obs \big) \equiv_{C_{i, t+1}} \big( \obsalt_i^{[1,t]} \cdot \obs \big) \end{align*}
    A clustering is incremental if it is incremental at each stage.
\end{definition}
    A clustering $C'$ is coarser than $C$ if each partition $C'_{i,t}$ is coarser than $C_{i,t}$, i.e.\ if each cluster in $C'_{i,t}$ is a union of clusters in $C_{i,t}$. We call $C'$ finer than $C$ if $C$ is coarser than $C'$.

\begin{definition}
We call a clustering $C'$ \emph{lossless} with respect to another clustering $C$, if $C'$ is coarser than $C$ and 
\[\max_{\policy \in \jointpolicies^{C'}} V_{\pi}(\initbelief, \horizon) = \max_{\policy \in \jointpolicies^C} V_{\pi}(\initbelief, \horizon),\]
where $\jointpolicies^{C'}$ and $\jointpolicies^C$ denote the set of all cluster policies corresponding to $C'$ and $C$, respectively.
\end{definition}
We call a clustering lossless, if it is lossless with respect to a trivial (i.e.\ no) clustering.

We define clustered sliding window memory recursively, stage by stage. We write $\jointcluster_{\neq i}^{t-1}$ for the tuple consisting of the cluster of each agent except agent $i$ in stage $t-1$. We write $\forget(\jointcluster_{\neq i}^{t-1}, \jointobs_{\neq i}^{t})$ for the resulting tuple of clusters of the other agents that we get in stage $t$ after applying sliding window memory (Def.\ \ref{def:sliding}). Formally, these are the clusters in the finest clustering which is incremental at stage $t-1$ and coarser than sliding window memory clustering.

\begin{definition} \label{def:belief-equivalence}
Two suffixes $\obs_i^{[\ell_t,t]}$ and $\obsalt_i^{[\ell_t,t]}$ are \emph{belief-equivalent},  written $\obs_i^{[\ell_t,t]} \sim_b \obsalt_i^{[\ell_t,t]}$, if for all $v \!\in\! \{\ell_t,\ldots,t\}$,
\[ \Pr\Big(s^t, \forget(\jointcluster_{\neq i}^{t-1}, \jointobs_{\neq i}^{t})\! ~\big|~  \!\obs_i^{[v,t]}\Big) = \Pr\Big(s^t, \forget(\jointcluster_{\neq i}^{t-1}, \jointobs_{\neq i}^{t})\!  ~\big|~  \!\obsalt_i^{[v,t]}\Big). \]
\end{definition}

For $v = \ell_t$, this definition states that the joint belief over the states and the clusters of the other agents is the same for  $\obs_i^{[\ell_t,t]}$ and $\obsalt_i^{[\ell_t,t]}$. 
Using a result of \citet{hansen2004dynamic} yields:
\begin{lemma}\label{lossless}
If a clustering is incremental, coarser than sliding $k$-window memory and finer than belief-equivalence, it is lossless w.r.t.\ sliding $k$-window memory.
\end{lemma}

Only demanding that these beliefs are equal for $v = \ell_t$ in Def.~\ref{def:belief-equivalence} does not give an incremental clustering (we give an example in App.~\ref{appA}). 
Instead, we need that the beliefs are equal for all $v \in \{\ell_t, \dots, t\}$. Intuitively, this means that the belief is also the same for the two suffixes when forgetting observations. 
Def.~\ref{def:belief-equivalence} alone is \emph{not} sufficient to establish incrementality of the clustering. However, when ensuring that each cluster contains precisely the LOHs with common suffix $\obs_i^{[m,t]}$, we can prove incrementality. We define these clusters using the following equivalence relation:
\begin{definition}
Consider suffixes $\obs_i^{[\ell_t,t]}$,  $\obsalt_i^{[\ell_t,t]}$ with largest common suffix $\obs_i^{[m,t]} = \obsalt_i^{[m,t]}$. They
are \emph{equivalent}, written $\obs_i^{[\ell_t,t]} \sim \obsalt_i^{[\ell_t,t]}$, if for all  $\obsaltalt_i^{[\ell_t,t]}$ with $\obsaltalt_i^{[m,t]} = \obs_i^{[m,t]}$, we have $\obs_i^{[\ell_t,t]} \sim_b\obsaltalt_i^{[\ell_t,t]}$.
\end{definition}
In this definition, we allow $m > t$, in which case the suffix is empty, and we cluster all LOHs of agent $i$ together.

We prove that $\sim$ is indeed an equivalence relation in App.~\ref{appA}. We write $[\obs_i^{[\ell_t,t]}]$ for the equivalence class of $\obs_i^{[\ell_t,t]}$. With this in hand, we can define the clustering corresponding to clustered sliding window memory.

\begin{definition}
      The clustering corresponding to clustered sliding window memory with window size $\windowsize$ consists of the sets defined by $C_{i,t} = \left\{[\obs_i^{[\ell_t,t]}] ~\Big|~\obs_i^{[\ell_t,t]} \in \observations_i^{\min\{t, k\}}\right\}$.
\end{definition}

Under this notion of equivalence, identical extensions of equivalent clusters are equivalent, which implies:

\begin{lemma}\label{incremental}
Clustered sliding window memory is incremental.
\end{lemma}

\noindent
Lemma \ref{lossless} and \ref{incremental}
 together imply:

\begin{theorem}
Clustered sliding window memory is lossless with respect to sliding window memory.
\end{theorem}

\paragraph{Probability-based clustering.}  To further limit the number of clusters, we can also cluster suffixes to a smaller common suffix together if the probability corresponding to the smaller common suffix is still smaller than some threshold $p_{\max}$. Formally, we consider two suffixes  $\obs_i^{[\ell_t,t]}$, $\obsalt_i^{[\ell_t,t]}$ with largest common suffix $\obs_i^{[m,t]} = \obsalt_i^{[m,t]}$ to be \emph{approximately equivalent} if $\obs_i^{[\ell_t,t]} \sim \obsalt_i^{[\ell_t,t]}$ or $\Pr(\obs_i^{[m,t]}) \leq p_{\max}$, and define the clusters \mbox{to be the equivalence classes of this equivalence relation.}

\section{Small-Step Multi-Agent \texorpdfstring{A$^{*}$}{A*}}

Our algorithm searches for a policy by incrementally fixing actions, for one clustered LOH at the time. Specifically, our algorithm uses small-step \maastar \cite{DBLP:conf/ijcai/Koops0JS23}, which in turn builds on \maastar \cite{szer2005maa}. We present the algorithm explicitly using clusters.
Small-step \maastar explores clusters in a fixed order: first by stage, then by agent, and then according to a given order of the clusters. 
\begin{definition} 
Given a total order $\preceq_{i,t}$ on $C_{i,t}$ for $t < \horizon$ and $i \in \agents$, we define an order $\preceq$ on $\bigcup_{i \in \agents} C_i$ by $\cluster_i^t \preceq \clusteralt_j^{t'}$ iff 
\[ \big(t < t'\big) \text{ or } \big(t = t' \land i < j\big) \text{ or } \big(t = t' \land i = j \land \cluster_i^t \preceq_{i,t} \clusteralt_j^{t'}\big).
\]
We call this order the \emph{expansion order}. 
\end{definition}

In \maastar and its derivatives, we incrementally construct policies. The intermediate policies are called \emph{partial}. In particular, a \emph{local partial policy} is a partial function $\ppol_i \colon C_i \!\rightarrow \actions_i$. A \emph{partial policy} is a tuple $\jointppol = \tuple{\ppol_1, \dots, \ppol_n}$ such that the local partial policies $\ppol_i$ are defined on precisely the clusters of agent $i$ among the first $d$ clusters in the expansion order for some $d$. Let $\ppolicies$ be the set of all partial policies. The \emph{stage $\stage{\jointppol}$  of $\jointppol$} is $u$ if $\jointppol$ is defined on all clusters of length $u-1$, but not on all clusters of length $u$.
A partial policy $\jointppol'$ \emph{extends} a partial policy $\jointppol$,  written $\jointppol \extendedby \jointppol'$, if $\jointppol'$ agrees with $\jointppol$ on all clusters on which $\jointppol$ is defined, formally:
\[\jointppol \extendedby \jointppol' \iff \forall i \in \agents. \forall\cluster_i \in C_i. \ppol_i(\cluster_i) \in \{ \bot, \ppol'_i(\cluster_i) \},\]
where $\ppol_i(\cluster_i) = \bot$ if $\ppol_i(\cluster_i)$ is not defined. The \emph{extensions of $\jointppol$} are the
 fully specified cluster policies extending~$\jointppol$, $ \ext{\jointppol} = 
\big\{ \jointpolicy \in \clusterpolicies \mid \jointppol \extendedby \jointclusterpol  \big\}$. Using this, we can define the small-step search tree:

\begin{definition}\label{def:search-tree}
The \emph{small-step search tree} for a Dec-POMDP is a tree whose nodes are the partial policies, the root node is the empty policy, and the children of a partial policy $\jointppol$ are exactly the partial policies $\jointppol'$ such that (1)~$\jointppol \extendedby \jointppol'$ and (2)~$\jointppol'$ is defined on one additional LOH compared to $\jointppol$.
\end{definition}

Small-step \maastar applies \astar to the small-step search tree. \astar expands nodes in a search tree guided by a heuristic $Q \colon \ppolicies \rightarrow \mathbb{R}$ \cite{DBLP:books/aw/RN2020}. The \astar algorithm keeps a priority queue of open nodes, and in each step, it expands the open node with the highest heuristic value by adding all children of that node to the queue. The algorithm terminates once a leaf is selected as node with the highest heuristic value. The algorithm finds a best clustered policy $\jointclusterpol $ if the heuristic is \emph{admissible}, i.e.\ an upper bound:
\[\textstyle
Q(\jointppol) \;\geq\; \max_{\jointclusterpol \in \ext{\jointppol}} V_{\jointclusterpol}(b, \horizon),\]
and matches the value for fully specified policies $\jointclusterpol$, i.e.\ $Q(\jointclusterpol) = V_{\jointclusterpol}(b, \horizon)$ for fully specified policies $\jointclusterpol$.

The algorithm is exact if in addition the clustering is lossless, i.e.\ $\max_{\jointpolicy \in \jointpolicies} V_{\jointpolicy}(b, \horizon) = \max_{\jointpolicy \in \clusterpolicies} V_{\jointpolicy}(b, \horizon)$.

 \newcommand{\progress}{\textit{prog}}

\section{Policy-Finding Multi-Agent \astar} \label{sec:loose_heuristics}
\label{sec:pfmaa}

In \maastar, to find the  optimal policy, one must expand all nodes in the queue whose heuristic value exceeds the value of the optimal policy. As a result, heuristics should be as tight as possible to limit the number of nodes expanded. 

With \emph{policy-finding multi-agent \astar} (\ourLB), we aim to find good policies fast. \ourLB applies small-step \maastar with clustered sliding window memory. To ensure timely termination,  we limit the number of policies expanded by pruning the priority queue. 
Furthermore, it is not essential that the \astar-heuristics for policy finding are tight. 
In this context, a useful heuristic is a heuristic that overestimates the value of good policies more than the value of bad policies, thereby steering the algorithm towards the good policies. To allow covering a larger fragment of the search space within a given time limit, the heuristics should also be easy to compute.

\paragraph{Finding good policies using \astar.}
In \ourLB, we prune the priority queue to find a good policy faster. In particular, we limit the number of policies expanded to $\horizon \cdot L$ for some $L$, while guaranteeing that we find a fully specified policy.
To do this, we define a progress measure $\progress$, and prune policies with a low progress.  That is, let $N$ denote the number of policies already expanded. We expand partial policy $\jointppol$ if $\progress(\jointppol) \geq N$ and prune it otherwise.  

Our progress measure $\progress$ satisfies two design goals. 
First, it limits the number of policies expanded up to \emph{each} level in the search tree, to ensure that the queue always contains a policy which is not pruned. This implies that we find a fully specified policy.
Second, it should prune the least promising partial policies. For this, note that admissible heuristics are more optimistic for partial policies for which fewer actions (or actions for clusters with lower probability) have been specified. Hence, a deep policy with some heuristic value is more likely to have a good policy as descendant than a shallow policy with the same heuristic value. Hence, $\progress$ should increase with the number of actions specified and the probability of the clusters for which an action is specified.

These design goals lead to the following progress measure.  Consider that  $\jointppol$ has specified an action for all clusters of length $\stage{\jointppol}-1$, for $i$ out of $n$ agents, and for $c$ out of $|C_{i+1,\stage{\jointppol}}|$ clusters of agent $i+1$. Let $p$ be the probability that the LOH of agent $i+1$ is in one of the first $c$ clusters. Then we define the \emph{progress} of $\jointppol$ as
\[
\progress(\jointppol) = \stage{\jointppol} \cdot L + i \cdot \tfrac{L}{n} + c + p \cdot\left(\tfrac{L}{n} - \left|C_{i+1,\stage{\jointppol}}\right| \right).
\]
We assume that  $L \geq n |C_{i, t}|$ for all $t < \horizon$ and all $i \in \agents$. In App.~\ref{appC}, we show that this progress measure indeed ensures that \ourLB finds a fully specified policy within $\horizon \cdot L$ policy expansions, and give further intuition.

\paragraph{Maximum reward heuristic.} We use the following simple \emph{maximum reward heuristic} $Q_{\textsf{maxr},r}(\jointppol)$, where $r > \stage{\jointppol}$. This heuristic computes a Dec-POMDP heuristic for $\jointppol$ with horizon $r$, and upper bounds the reward over the remaining $\horizon-r$ stages by the maximum reward (over all state-action pairs) per stage. 
Although this heuristic is clearly not tight in general, it is still useful when finding lower bounds: search guided by this heuristic is essentially a local search, choosing policies yielding good reward over the next $r - \stage{\jointppol}$ stages.

\paragraph{Terminal reward MDP heuristic.} To take into account the effect of actions on later stages to some extent, we can also use the \emph{terminal reward MDP heuristic}  $Q_{\textsf{MDP},r}(\jointppol)$, where $r > \stage{\jointppol}$. This is a Dec-POMDP heuristic for $\jointppol$ with horizon $r$ for a Dec-POMDP with terminal rewards, where these terminal rewards represent the MDP value for the remaining $\horizon-r$ stages. Formally, if $Q_{\textsf{MDP}}(s, \horizon')$ is the optimal value of the corresponding MDP with initial state $s$ and horizon $\horizon'$, then the terminal reward corresponding to a joint belief $b$ computed from the joint observation history at stage $r$ is
\begin{equation} \label{eq:terminalMDP}
\textstyle
\sum_{s \in \states} b(s) \cdot Q_{\textsf{MDP}}(s, \horizon-r).
\end{equation}
We then take the weighted average over all joint beliefs to compute the terminal reward.

\paragraph{Horizon reduction.} Heuristic values are computed recursively, similarly to the small-step \maastar implementation \rsmaa. To avoid having to compute heuristics for large horizons, \ourLB first applies a \emph{horizon reduction}, reducing the computation of a heuristic for a horizon $\horizon'$ Dec-POMDP to a computation for horizon $r$ Dec-POMDP with terminal rewards representing the remaining  $\horizon'-r$ stages. For \ourLB, this terminal reward is an MDP value or just $\horizon'-r$ times the maximal reward, as explained above.

\section{Terminal Reward Multi-Agent \astar} \label{sec:terminal_heuristic}

\newcommand{\statetrace}{\bm{s}}
\newcommand{\statetracealt}{\widetilde{\bm{s}}}
\newcommand{\constraint}{\mathcal{I}}
\newcommand{\constraintrelax}{\constraint'}

Among the heuristics used so far in the literature, only the MDP heuristic~\cite{DBLP:conf/icml/LittmanCK95} can be computed effectively for large horizons. This heuristic reveals the state to the agents in each step, and thereby overapproximates the value of a belief node drastically. 
The POMDP heuristic \cite{szer2005maa,DBLP:conf/atal/RothSV05} is tighter as it assumes that each agent receives the full joint observation, but the state information is not revealed.
The recursive heuristics considered by \citet{DBLP:conf/ijcai/Koops0JS23} reveal the joint observation only once during planning.
Both heuristics are too expensive to compute for large horizons on a variety of benchmarks (as remarked by \citet{oliehoek2013incremental} for the POMDP heuristic). 

We propose a computationally more tractable alternative that periodically reveals the state. This leads to a new family of admissible heuristics, which we call \emph{terminal reward heuristics}, which are empirically tighter than the POMDP heuristic, but computationally cheaper. Applying small-step \maastar with lossless clustering and this heuristic yields \emph{terminal reward multi-agent \astar} (\ourUB). With \ourUB, we aim to find a tight upper bound for the value of an optimal policy, which is also scalable.

\paragraph{Generalized policies.} To define the heuristics, we introduce a more general type of policy. Formally, a \emph{generalized local policy} $\policy_i \colon (\jointobservations \times \states)^{\leq h-1} \rightarrow \actions_i$ maps histories of states and joint observations to a local action. 
A \emph{generalized (joint) policy} is a tuple $\tuple{\pi_1, \ldots, \pi_n}$ of generalized local policies\footnote{Including past observations states seems unnecessary for policies that have access to the current state, however, we will add constraints that the policy cannot depend on the most recent state.}. Let $\genpolicies$ denote the set of all generalized policies. 
As before, a generalized policy $\jointpolicy$ \emph{extends} a partial policy $\jointppol$, denoted by $\jointppol \extendedby \jointpolicy $, if $\policy_i$ ignores the state information and agrees with $\ppol_i$ on all OHs corresponding to an LOH for which $\ppol_i$ specifies the action. Formally, $\jointppol \extendedby \jointpolicy$ iff
$\forall i \in \agents. \forall (\jointobstrace,\statetrace) \in (\jointobservations\times\states)^{\leq h-1}.~  \ppol_i(\obstrace_i) \in \{ \bot, \policy_i(\jointobstrace, \statetrace) \}$. 
Let $\extgen{\jointppol}  = 
\big\{ \jointpolicy \in \genpolicies \mid \jointppol \extendedby \jointpolicy  \big\}$ be the set of generalized extensions of a partial policy $\jointppol$ agreeing with $\jointppol$.

\paragraph{Terminal reward heuristic.} We write the Dec-POMDP optimization problem over generalized policies as
\begin{equation}
\begin{aligned}
&\max_{\jointpolicy \in \genpolicies}  V_{\jointpolicy}\left(b, \horizon\right) \\ &~ \quad\!\!\text{subject to} \quad \obstrace_i = \obstracealt_i \quad\!\!\text{implies}\quad\!\!  \policy_i\left(\jointobstrace, \statetrace\right) = \policy_i\left(\jointobstracealt, \statetracealt\right)  \\ &~
\quad\text{for all } (i,\jointobstrace, \jointobstracealt, \statetrace, \statetracealt)  \in \mathcal{I}, \text{ with} \\
\constraint  &= \bigcup_{v=0}^{\horizon-1} \Big\{(i,\jointobstrace, \jointobstracealt, \statetrace, \statetracealt) \mid i \in \agents, \;\jointobstrace, \jointobstracealt \in \jointobservations^v, \;\statetrace, \statetracealt \in \states^v\Big\}.\!\!\!\!\!\!\!\!
\end{aligned}
\label{eq:opt_generalized_policies}
\end{equation}
Using $\constraint$, we quantify over all agents and pairs of traces. The condition $\obstrace_i = \obstracealt_i$ selects the pairs that agent~$i$ cannot distinguish. On these pairs, the agent has to take the same action.
To compute a heuristic $Q(\jointppol)$, instead of maximizing over $\genpolicies$ in Eq.~\eqref{eq:opt_generalized_policies}, we  maximize over the generalized extensions $\extgen{\jointppol}$ of $\jointppol$.
Moreover, we relax some of these constraints, i.e.\ we only consider a subset $\constraintrelax \subseteq \constraint$ of the constraints. Denote the corresponding heuristic by $Q_{\constraintrelax}(\jointppol)$. Since relaxing constraints can only increase the maximum, we get:

\begin{theorem}
$Q_{\constraintrelax}$ is admissible for all  $\constraintrelax \subseteq \constraint$. 
\end{theorem}

It remains to explain which constraints to relax to obtain a heuristic which is easier to compute. This is illustrated in Figure \ref{fig:terminal_reward} on the left. We propose to relax, what we call, \emph{late constraints}, i.e., we allow policies to depend on the $r$th state from some stage $r$ onwards.  We can then split the original problem with horizon $\horizon$ as a Dec-POMDP with horizon $r{<}h$ and a terminal reward representing the remaining $\horizon-r$ stages. The terminal reward is a Dec-POMDP value, with horizon $\horizon-r$ and the revealed state as initial belief. This is similar to Eq.\ \eqref{eq:terminalMDP}, but using Dec-POMDP values instead of MDP values.  Since we reveal the state, the agents can use more information to decide which actions they take. This therefore gives an upper bound. Choosing a larger $r$ typically yields a tighter, but more expensive heuristic. 

\begin{figure}[tbp]
    \centering
    \resizebox{.733\columnwidth}{!}{\def\svgwidth{207pt}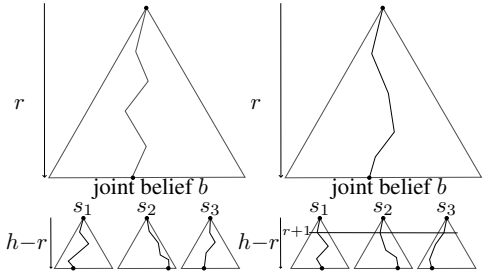}
    \caption{Revealing the state for a joint belief $b$ at stage $r$ over three states $s_1$, $s_2$, $s_3$. On the left the state is revealed at stage $r$, on the right it is revealed at stage $r+1$. In the latter case, the three policies are forced to take same action in stage $0$. }
    \label{fig:terminal_reward}
\end{figure}

\paragraph{Revealing the state at stage $r+1$.}  We now explain how to compute a tighter variant of this heuristic, where we allow policies to depend on the $r$th joint belief from stage $r$ onwards and on the $r$th state from stage $r+1$ onwards. This is illustrated in Figure~\ref{fig:terminal_reward} on the right. We consider this problem as a Dec-POMDP with horizon $r$ and a terminal reward depending on the joint belief. Since we reveal the joint belief $b$ in stage $r$, we can compute each of these terminal rewards separately. We compute for each state $s$ and each joint action $\jointact$, a heuristic $Q(s, \jointact, \horizon-r)$ for the Dec-POMDP with horizon $\horizon-r$ where the initial state is $s$ and the agents take action $\jointact$ in stage $0$. To model that the state information cannot be used in stage $0$, we now compute 
\begin{equation}\textstyle\max_{\jointact \in \jointactions}~ \sum_{s \in \states}\; b(s)\cdot Q(s, \jointact, \horizon-r) \label{eq:terminalDecPOMDP}
\end{equation} as a heuristic for the terminal reward corresponding to the joint belief $b$ revealed in stage $r$. Taking the maximum over the actions outside of the weighted sum in Eq. \eqref{eq:terminalDecPOMDP}  ensures that the same action is taken in stage $0$, i.e.\ that the state information is not used in stage $0$ (which is stage $r$ of the original Dec-POMDP).
In practice, since computing $Q(s, \jointact, \horizon-r)$ is relatively expensive, we first compute an MDP heuristic for the Dec-POMDP with horizon $\horizon-r$, and only compute  $Q(s, \jointact, \horizon-r)$ for the actions $\jointact \in \jointactions$ which based on the MDP values can possibly attain the maximum in Eq. \eqref{eq:terminalDecPOMDP}.

\paragraph{Horizon reduction.} \ourUB also applies a horizon reduction as explained in Sec.\ \ref{sec:loose_heuristics}. For \ourUB, the terminal reward is a heuristic computed as explained in the previous paragraph. If $\horizon > r$, we also apply the horizon reduction to compute a heuristic for the root node.

\section{Empirical Evaluation}\label{sec:empirical}

 This section provides an empirical evaluation of \ourLB (Sec.~\ref{sec:pfmaa}) and \ourUB (Sec.~\ref{sec:terminal_heuristic}). As baseline, we provide reference values from the state-of-the-art $\epsilon$-optimal solver \textsc{FB-HSVI} \cite{DibangoyeABC14,DBLP:journals/jair/DibangoyeABC16} (for $\epsilon=0.01$) and the state-of-the-art approximate solver, which is the genetic algorithm \textsc{GA-FSC} \cite{DBLP:journals/aamas/EkerA13}\footnote{The implementations for these solvers are not available, we copy the achieved results from the respective papers.}, and upper bounds found by \rsmaa \cite{DBLP:conf/ijcai/Koops0JS23}.
For \rsmaa, we give the best results among the heuristics $Q_{1, \infty}$, $Q_{25, \infty}$, $Q_{200, \infty}$ and $Q_{200, 3}$. 

\paragraph{Implementation.}
We base our implementation on the data structures and optimizations of \rsmaa \cite{DBLP:conf/ijcai/Koops0JS23}\footnote{See \url{https://zenodo.org/records/11160648}.}. Notably, we handle the last agent's last stage efficiently and abort heuristic computations after $M=200$ steps. 

\paragraph{Setup.} All experiments ran on a system with an Apple M1 Ultra using the PyPy environment. We ran \ourLB for six configurations: three configurations aimed at finding policies fast and limited to 60 seconds of CPU time each, and three configurations aimed at finding high-quality policies limited to 600 seconds. We ran \ourUB with a limit of 120 seconds and a limit of 1800 seconds, respectively. We used a global 16GB  memory limit. We validated \ourLB values reported in this paper with a separate code base. We gave both \rsmaa and \ourUB the values found by \ourLB.

\paragraph{Hyperparameter selection.} 

\begin{table}\centering
\resizebox{\columnwidth}{!}{
\begin{tabular}{@{}ccc||ccc||ccc@{}} \hline
\multicolumn{3}{c||}{comparison} & \multicolumn{3}{c||}{fast results} & \multicolumn{3}{c}{high-quality results} \\
heur. & $k$ & $L$ & heur. & $k$ & $L$ & heur. & $k$ & $L$ \\ \hline
$Q_{\textsf{MDP}}$ & 2 & 1000 & $Q_{\textsf{MDP},1}$ & 1 & 20 & $Q_{\textsf{MDP},2}$ & 2 & 1000 \\
$Q_{\textsf{MDP}}$ &  3 & 10\,000 &$Q_{\textsf{MDP},1}$ & 2 & 100 &  $Q_{\textsf{MDP},2}$ & 3 & 10\,000 \\ & &  & $Q_{\textsf{MDP},1}$ & 3 & 100 & $Q_{\textsf{maxr},3}$ & 3 & 10\,000 \\ \hline
\end{tabular}
}
\caption{Hyperparameters for \ourLB.}\label{tab:hyper}
\end{table}

For \ourLB, the main hyperparameters are the window size $\windowsize$, which of the heuristics $Q_{\mathsf{maxr}, r}$ or $Q_{\mathsf{MDP}, r}$ to use, the depth $r$ of the heuristic, and the iteration limit $L$ per stage. We report the configurations that we used in Table~\ref{tab:hyper}. Setting $r \geq 4$ or $k \geq 4$ is not feasible for all benchmarks. For \ourUB, we use the heuristic from Section \ref{sec:terminal_heuristic} with $r=3$ and $r=5$ respectively.  In each case, the $Q_3$ heuristic is used to solve Dec-POMDPs.

\paragraph{Benchmarks.}

 We used the standard benchmarks from the literature:
 \dectiger{}~\cite{nair2003taming}, \ff{}~\cite{oliehoek2008optimal} (3 fire levels, 3 houses), \grid{} with two observations~\cite{amato2006optimal}, \boxpush~\cite{DBLP:conf/uai/SeukenZ07}, \gridthree~\cite{amato2009incremental}, \mars~\cite{amato2009achieving}, \hotel~\cite{spaan2008interaction}, \recycling~\cite{amato2007optimizing}, and \broadcast~\cite{hansen2004dynamic}.

\newcommand{\highlightt}[1]{\textcolor{green!60!black}{\textbf{#1}}}

\begin{table}[t]
\resizebox{\columnwidth}{!}{
\begin{tabular}{@{}r|rr|r|rr|rr@{}} \hline
 $h$ & FB-HSVI & GA-FSC & \multicolumn{3}{|c|}{\ourLB} & random & upper \\ \hline
    &          &        & $Q_{\mathsf{MDP}}$ & fast & quality & & \\ \hline
\multicolumn{8}{c}{ \dectiger } \\ \hline
50 & 80.7 &  & 79.9 & 79.1 & \highlightt{81.0} &  \textminus 2311.1 & 101.3 \\ 
100 & \highlightt{170.9} & 169.3 & 169.3 & 169.3 & \highlightt{170.9} &  \textminus 4622.2 & 206.4 \\ \hline
\multicolumn{8}{c}{ \grid } \\ \hline
20 &  &  & 14.23 & 14.36 & \highlightt{14.69} & 4.67 & 17.13 \\ 
50 &  & \highlightt{40.49} & 36.86 & 37.46 & MO & 12.17 & 47.21 \\ \hline
\multicolumn{8}{c}{ \boxpush } \\ \hline
20 & 458.1 & 468.1 & 402.5 & 466.3 & \highlightt{475.0} &  \textminus 20.5 & 476.4 \\ 
50 & 1134.7 & 1201.0 & 949.8 & 1207.7 & \highlightt{1209.8} &  \textminus 57.9 & 1218.4 \\ 
100 &  & 2420.3 & 1864.5 & 2431.4 & \highlightt{2433.5} &  \textminus 120.5 & 2453.4 \\ \hline
\multicolumn{8}{c}{ \mars } \\ \hline
50 & \highlightt{128.9} &  & 116.7 & 117.0 & 122.6 &  \textminus 62.0 & 132.8 \\ 
100 & \highlightt{249.9} &  & 221.7 & 222.0 & 234.1 &  \textminus 122.7 & 265.7 \\ \hline
\end{tabular}
}
\caption{Lower bounds, i.e.\ the value of the policies obtained with different policy-finding algorithms. MO denotes memout (\textgreater 16GB). A full version is available in App.~\ref{appE}.}  \label{tab:resultsLB}
\end{table}

\subsubsection{Results}
Before we discuss the results, we present an overview of the most interesting data in 
Tables~\ref{tab:resultsLB} and~\ref{tab:resultsUB}, giving results for the lower and upper bound, respectively. 
Further data is given in App.~\ref{appE}. Each table presents results for various benchmarks and different horizons $h$. In Table~\ref{tab:resultsLB}, we first give the results for the baselines FB-HSVI and GA-FSC. 
The next three columns give results for \ourLB, using either the $Q_{\textsf{MDP}}$ heuristic (as comparison),  all fast configurations, or all  high-quality configurations.
For reference, we also give the value for the random policy (i.e.\ the policy that uniformly randomizes over all actions) and the best-known \emph{upper bound}\footnote{We took these bounds from either the literature or our methods, see App.~\ref{appE} for details.}.  
In Table~\ref{tab:resultsUB}, we give FB-HSVI's result (if $\epsilon$-optimal) and results for \rsmaa and \ourUB for time limits of 120s and 1800s, respectively. Finally, we give the best-known \emph{lower bound} (from a known policy) and the MDP value, which is a trivial upper bound.

\paragraph{State-of-the-art policies.} Using \ourLB, we compute policies comparable with or better than the state-of-the-art for most benchmarks.  In particular, for \boxpush{}, we provide the best known policies for $h \geq 10$. Although the improvement over GA-FSC in absolute sense is not large, the upper bound shows that our improvement is a significant step towards optimality.  On \dectiger{}, we find better solutions (with $k=3$) than FB-HSVI and GA-FSC. It is known that with $k=4$, better policies exist for e.g.\ $h=50$. However, the policy space with $k=4$ is so large that \ourLB does not find the (almost) optimal policies in that space.

\paragraph{Novel heuristics outperform $Q_{\mathsf{MDP}}$.} Our algorithm \ourLB consistently finds better policies using the novel heuristics $Q_{\mathsf{maxr}, r}$ and $Q_{\mathsf{MDP}, r}$ from Sec.\ \ref{sec:loose_heuristics} than using $Q_{\mathsf{MDP}}$ (for the same number of iterations $L$; see Table~\ref{tab:hyper}). We note that $Q_{\mathsf{MDP}}$ is the only heuristic in the literature that scales up to high horizons on all benchmarks.

\paragraph{Fast versus high-quality configurations.} Even the fast configuration typically yields policies that are reasonably close to the best-known policy. Nevertheless, the high-quality configuration is always able to find (slightly) better policies.

\paragraph{Challenges for \ourLB.} 
 On \grid{}, the best policies for large horizons require actions that are sub-optimal in the short run, and are hence hard to find using \ourLB. 
On \mars{}, FB-HSVI yields better results than \ourLB. We conjecture that this is due to the larger search space for \mars{}. To reduce the search space, we used probability-based clustering with $p_{\max} = 0.2$; this improves the values to 125.84 and 243.97 for $h=50$ and $h=100$ respectively.

\begin{table}[t]
\resizebox{\columnwidth}{!}{%
\begin{tabular}{@{}r|r|rr|rr|rr@{}} \hline
$h$ & FB-HSVI & \multicolumn{2}{|c|}{\rsmaa} & \multicolumn{2}{|c|}{\ourUB} & lower & MDP \\ \hline
    &          &  fast & quality & fast & quality & & \\ \hline
\multicolumn{8}{c}{ \grid } \\ \hline
7 & 4.48 & TO & MO & \highlightt{4.47} & 4.49 & 4.47 & 5.81 \\  
20 &  & TO & MO & \highlightt{17.13} & MO & 14.69 & 18.81 \\ 
50 &  & TO & MO & \highlightt{47.21} & MO & 40.49 & 48.81 \\ \hline
\multicolumn{8}{c}{ \boxpush } \\ \hline
20 &  & TO & MO & 481.2 & \highlightt{476.4} & 475.0 & 511.1 \\ 
50 &  & TO & MO & 1227.4 & \highlightt{1218.4} & 1209.8 & 1306.2 \\ 
100 &  & TO & MO & 2469.7 & \highlightt{2453.4} & 2433.5 & 2628.1 \\ \hline
\multicolumn{8}{c}{ \mars } \\ \hline 
50 &  & \highlightt{132.8} & \highlightt{132.8} & 136.5 & 136.9 & 128.9 & 145.0 \\ 
100 &  & 287.5 & \highlightt{265.7} & 273.4 & 276.5 & 249.9 & 289.0 \\ \hline
\end{tabular}
}
\caption{Upper bounds on the value of optimal policies, obtained with algorithms that find such bounds. TO and MO denote timeout (\textgreater 120s for the fast configuration) and memout (\textgreater 16GB). A full version is available in App.~\ref{appE}.}  \label{tab:resultsUB}
\end{table}

\paragraph{A scalable upper bound.} On all benchmarks, we provide an upper bound for horizons up to 100 which is better than $Q_{\textsf{MDP}}$. This is an improvement over \rsmaa, which only reaches horizon 10 on \boxpush{} and horizon 6 on \grid{}. In fact, for  larger horizons, \rsmaa{} cannot even compute  the cheapest recursive heuristic, $Q_{1, \infty}$, within the memory limit.  
On \boxpush{}, the bound is also tight. It is at least 9 times closer to the optimum than the MDP value. 

\paragraph{Revealing the state is too optimistic.} When \rsmaa can compute an upper bound, \ourUB generally gives worse upper bounds than \rsmaa with a cheap heuristic.

\paragraph{Timings.} \ourLB is generally faster than required by the time limit. For instance, \boxpush{} $h=100$ takes 9 and 401 seconds combined, respectively. For $r=5$, \ourUB often reaches its memory limit within 300 seconds.

\section{Conclusion}

We presented novel methods for finding
good policies for Dec-POMDPs and for finding upper bounds on the optimal value. Together, this allows us to find policies with values that are at least 99\% of the optimal value, on horizons an order of magnitude higher than for which exact solving is possible, as well as giving upper bounds for high horizons the first time. The main advancements leading to this result are clustered sliding window memory, pruning the priority queue and several new \astar-heuristics. Further research includes investigating different methods for clustering observation histories.

\newpage

\section*{Acknowledgements}

We would like to thank the anonymous reviewers for their useful comments.
This work has been partially funded by the ERC Starting Grant DEUCE (101077178), the NWO Veni grant ProMiSe (222.147), and the NWO grant PrimaVera (NWA.1160.18.238).

\section*{Contribution Statement}
Wietze Koops is the primary designer of the algorithm and implemented it and evaluated the algorithms. The other authors contributed discussions, ideas, and shaped the description of the algorithm.

\bibliographystyle{named}
\bibliography{references}

\setcounter{lemma}{0}
\setcounter{theorem}{0}
\renewcommand{\baselinestretch}{1.14}
\renewcommand{\arraystretch}{0.9}
\selectfont

\newcommand{\jointmultiagentbelief}{\bm{\gamma}}

\appendix

\noindent {\huge \textbf{Appendix}} \medskip

\noindent This appendix includes:
\begin{itemize}
\item Appendix \ref{appA}: Proofs for the lemmas and claims from the section on clustering (Sec.~\ref{sec:clustering}), and an empirical evaluation of clustered sliding window memory.
\item Appendix \ref{appB}: Pseudo-code of \ourLB (Sec.~\ref{sec:clustering}).
\item Appendix \ref{appC}: Additional explanation on queue pruning used in \ourLB (Sec.~\ref{sec:loose_heuristics}).
\item Appendix \ref{appD}: Additional explanation on the concept of horizon reduction (Sec.~\ref{sec:loose_heuristics} and Sec.~\ref{sec:terminal_heuristic}). 
\item Appendix \ref{appE}: Additional information on the Empirical Evaluation (Sec.~\ref{sec:empirical}), including information on hyperparameter selection, further discussion of the results and further tables with results. 
\end{itemize}

\section{Clustering}\label{appA}

In this section we provide proofs of the lemmas and claims from the section on clustering (Sec.~\ref{sec:clustering}), as well as an empirical evaluation of clustered sliding window memory.

\subsection{Proofs}

\subsubsection{Proof of Lemma \ref{lossless}}

To prove Lemma \ref{lossless}, we first introduce the \emph{multi-agent belief state} and the \emph{multi-agent belief} \cite{hansen2004dynamic}. 

\begin{definition}
A \emph{multi-agent belief state} for agent $i$ and stage $t$ is a pair $(s, \jointmultiagentbelief_{\neq i}^{h-t})$, where $s$ is a state and \[\jointmultiagentbelief_{\neq i}^{h-t} = \tuple{\gamma_{1}^{h-t}, \ldots, \gamma_{i-1}^{h-t}, \gamma_{i+1}^{h-t}, \ldots, \gamma_{n}^{h-t}}\] represents a policy for all agents except agent $i$ from stage $t$ onwards, where $\gamma_{j}^{h-t} \colon \observations_j^{\leq \horizon - t} \rightarrow \actions_j$ describes how agent $j$ acts given the observations they receive from stage $t$ onwards.

A \emph{multi-agent belief}  for agent $i$ and stage $t$ is a distribution over multi-agent belief states  for agent $i$ and stage $t$.
\end{definition}

Lemma \ref{lossless} in essence states that two suffixes $\obs_i^{[\ell_t,t]}$, $\obsalt_i^{[\ell_t,t]}$ can be clustered losslessly if they are belief-equivalent.
To prove Lemma \ref{lossless}, we use a result of \citet{hansen2004dynamic}, which states that we can losslessly cluster two nodes (in our context, suffixes) if the multi-agent beliefs are equal. To use this result, we first show that the condition in belief-equivalence (Def.~\ref{def:belief-equivalence}) for $v = \ell_t$, i.e.\ that 
\begin{equation}\label{eq:belief}
\begin{split} 
\Pr\Big(s^t, \forget(\jointcluster_{\neq i}^{t-1}, &\jointobs_{\neq i}^{t}) ~\big|~  \obs_i^{[\ell_t,t]}\Big) \\ &= \Pr\Big(s^t, \forget(\jointcluster_{\neq i}^{t-1}, \jointobs_{\neq i}^{t})  ~\big|~  \obsalt_i^{[\ell_t,t]}\Big), 
\end{split}
\end{equation}
implies that the multi-agent belief is the same for the two suffixes $\obs_i^{[\ell_t,t]}$ and $\obsalt_i^{[\ell_t,t]}$. We prove this in two steps.

\begin{claim} \label{cluster_claim1}
Let $C$ be a clustering which is incremental and coarser than sliding $k$-window memory clustering. Let $\jointcluster_{\neq i}^{v}$ denote the tuple of clusters of the agents except agent $i$ in stage $v$.
If two suffixes of agent $i$ induce the same distribution over pairs $(s, \forget(\jointcluster_{\neq i}^{t-1}, \jointobs_{\neq i}^{t}))$, then they induce the same distribution over the pairs $(s, \jointcluster_{\neq i}^t)$.
\end{claim}

\begin{proof}
By definition, the clusters in $\forget(\jointcluster_{\neq i}^{t-1}, \jointobs_{\neq i}^{t})$ are the clusters in the finest clustering $C^F$ which is incremental at stage $t-1$ and coarser than sliding $k$-window memory clustering. Since $C$ is  incremental at stage $t-1$ and coarser than sliding $k$-window memory clustering, it follows that $C^F$ is finer than $C$, so clusters in $C$ are unions of clusters in $C^F$.

Hence, we can compute the distribution over the pairs $(s, \jointcluster_{\neq i}^t)$ from the distribution over the pairs $(s, \forget(\jointcluster_{\neq i}^{t-1}, \jointobs_{\neq i}^{t}))$ by aggregating over all tuples of clusters in $C^F$ that are clustered to the same cluster in $C$, which implies the claim.
\end{proof}

\begin{claim} \label{cluster_claim2}
Assume that the clustering is incremental. If two suffixes of agent $i$ induce the same distribution over pairs $(s, \jointcluster_{\neq i}^t)$, then they induce the same multi-agent belief.
\end{claim}

\begin{proof} Throughout the proof, fix a joint policy $\jointpolicy$.\footnote{At the moment the clustering is computed, the algorithm has not yet specified actions beyond stage $t$. However, the argument holds for any (possible) joint policy $\jointpolicy$, and the agents do have full knowledge of $\jointpolicy$ when actually executing the policy.} We first show that we can compute the multi-agent belief state corresponding to a single pair $(s, \jointcluster_{\neq i}^t)$. We do this per agent $j \neq i$. Let $\cluster_j^t$ be the current cluster of agent $j$ and let $o_j^{[1,t]} \in \cluster_j^t$ be any observation history in that cluster. Then $\gamma_{j}^{h-t} \colon \observations_j^{\leq \horizon - t} \rightarrow \actions_j$ satisfies $\gamma_{j}^{h-t}(o_j^{[t+1,v]}) = \policy_j(o_j^{[1,v]})$. 

Since the clustering is incremental, the result does not depend on which $o_j^{[1,t]} \in \cluster_j^t$ we take. Hence, we can compute $\gamma_{j}^{h-t}$ if we know $\cluster_j^t$. This implies that we can compute the  multi-agent belief state $(s, \jointmultiagentbelief_{\neq i}^{h-t})$ from $(s, \jointcluster_{\neq i}^t)$.

It follows that if we have a distribution over pairs $(s, \jointcluster_{\neq i}^t)$, then we can compute the corresponding distribution over pairs  $(s, \jointmultiagentbelief_{\neq i}^{h-t})$, i.e.\ the corresponding multi-agent belief.
\end{proof}

\noindent With these claims in hand, we prove Lemma \ref{lossless}.

\begin{lemma}
If a clustering is incremental, coarser than sliding $k$-window memory and finer than belief-equivalence, it is lossless w.r.t.\ sliding $k$-window memory.
\end{lemma}
\begin{proof}
Let $C$ be a clustering satisfying the conditions of the lemma. 
Fix an agent $i$, a stage $t$, and a cluster $c_i^t \in C_{i,t}$. 

Let $\obs_i^{[\ell_t,t]},  \obsalt_i^{[\ell_t,t]} \in c_i^t$ be given. 
Since $C$ is finer than belief-equivalence, the fact that $\obs_i^{[\ell_t,t]}$  and $\obsalt_i^{[\ell_t,t]}$ are in the same cluster implies that $\obs_i^{[\ell_t,t]} \sim_b \obsalt_i^{[\ell_t,t]}$. By the definition of belief-equivalence, this implies that  $\obs_i^{[\ell_t,t]}$ and $ \obsalt_i^{[\ell_t,t]}$  induce the same distribution over pairs $(s, \forget(\jointcluster_{\neq i}^{t-1}, \jointobs_{\neq i}^{t}))$. 

By Claim \ref{cluster_claim1} and \ref{cluster_claim2}, it follows that $\obs_i^{[\ell_t,t]}$ and $ \obsalt_i^{[\ell_t,t]}$  induce the same multi-agent beliefs. By the result of \citet{hansen2004dynamic} that two suffixes can be clustered losslessly if the multi-agent beliefs are equal, this implies that $\obs_i^{[\ell_t,t]}$ and $\obsalt_i^{[\ell_t,t]}$ can be clustered together losslessly.
\end{proof}

\subsubsection{Example: Only Eq.\ \eqref{eq:belief} is not sufficient for incrementality}

We now give an example that shows that using only Eq.\ \eqref{eq:belief} (i.e., only demanding that the beliefs are equal for $v = \ell_t$ in Def.~\ref{def:belief-equivalence}) is not sufficient to obtain an incremental clustering.

\begin{example}
Consider a Dec-POMDP with three local observations, $X$, $Y$, and $Z$ where LOHs $XY$ and $XZ$ induce the same belief, but $YYY$ and $YZY$ induce different beliefs. We use sliding windows of size $k=2$.  If we cluster $XY$ and $XZ$, then an incremental clustering requires that we also cluster $XYY$ and $XZY$, so the suffixes $YY$ and $ZY$ are clustered after applying sliding window memory (Def.~\ref{def:sliding}).
However, $YY$ and $ZY$ do not yield the same belief, since $YYY$ and $YZY$ do not yield the same belief. \\ 
\end{example}

\subsubsection{Auxiliary Claims}

We proceed by proving some small auxiliary claims. 
We start by proving that \emph{the finest} clustering which is incremental at stage $t-1$ and coarser than sliding $k$-window memory clustering (which defines  $\forget(\jointcluster_{\neq i}^{t-1}, \jointobs_{\neq i}^{t})$)  exists.

\begin{claim} Fix a clustering $C$ up to stage $t-1$. Then there exists a \emph{finest} clustering $C^F$ up to stage $t$ agreeing with $C$ up to stage $t-1$, which is incremental at stage $t-1$ and coarser than sliding $k$-window memory clustering. 
\end{claim} 

\begin{proof} Fix an agent $i$ and let $\equiv_{C'_{i,t}}$ denote the equivalence relation corresponding to a partition $C'_{i,t}$. Let  $\equiv_{k}$ denote the equivalence relation corresponding to sliding $k$-window memory clustering. 

Let $\mathcal{C}$ be the collection of all clusterings $C'$ agreeing with $C$ up to stage $t-1$ which are incremental at stage $t-1$ and coarser than sliding $k$-window memory clustering.   Define the relation $\equiv_F$ by \[\obs_i^{[1,t]} \equiv_{F} \obsalt_i^{[1,t]} \iff \forall C' \in \mathcal{C}: \obs_i^{[1,t]} \equiv_{C'_{i,t}} \obsalt_i^{[1,t]}. \]
Then $\equiv_F$ is an equivalence relation inducing a partition $C^F_{i,t}$. 

If $\obs_i^{[1,t]} \equiv_{k} \obsalt_i^{[1,t]}$, then $\obs_i^{[1,t]} \equiv_{C'_{i,t}} \obsalt_i^{[1,t]}$ for all $C' \in \mathcal{C}$ (since each $C'$ is coarser), so also $\obs_i^{[1,t]} \equiv_{F} \obsalt_i^{[1,t]}$, so $C^F_{i,t}$ is coarser than sliding $k$-window memory clustering. 

Now we show incrementality. Let $\obs_i^{[1,t-1]} \equiv_{C_{i, t-1}} \obsalt_i^{[1,t-1]}$ and let $o$ be an additional observation, then \[\forall C' \in \mathcal{C} : \big( \obs_i^{[1,t-1]} \cdot \obs \big) \equiv_{C_{i, t}} \big( \obsalt_i^{[1,t-1]} \cdot \obs \big)\] since each $C'$ is incremental at stage $t-1$. Hence, we have $\big( \obs_i^{[1,t-1]} \cdot \obs \big) \equiv_{F} \big( \obsalt_i^{[1,t-1]} \cdot \obs \big)$, so $C^F$ is also incremental at stage $t-1$, as required.

Finally, we note that from the definition of  $\equiv_F$ it directly follows that $C^F$ is finer than each $C'$, so $C^F$ is the finest clustering which is incremental at stage $t-1$ and coarser than sliding $k$-window memory clustering.
\end{proof}

\noindent Next, we show that $\sim$ is an equivalence relation. 

\begin{claim}\label{claim:eqrel}
$\sim$ is an equivalence relation. 
\end{claim} 
\begin{proof}
Note that $\sim_b$ is clearly an equivalence relation. 
Throughout the proof, fix an agent $i$ and a stage $t$.

We first show that $\sim$ is reflexive. Let $\obs_i^{[\ell_t,t]}$ be given. We need to show that $\obs_i^{[\ell_t,t]} \sim \obs_i^{[\ell_t,t]}$. Since the largest common suffix of $\obs_i^{[\ell_t,t]}$ and $\obs_i^{[\ell_t,t]}$ is $\obs_i^{[\ell_t,t]}$, we only have to show that $\obs_i^{[\ell_t,t]} \sim_b\obs_i^{[\ell_t,t]}$, which holds since $\sim_b$ is reflexive. 

We now show that $\sim$ is symmetric. Let $\obs_i^{[\ell_t,t]}$ and $ \obsalt_i^{[\ell_t,t]}$ with $\obs_i^{[\ell_t,t]} \sim \obsalt_i^{[\ell_t,t]}$ be given. Let  $\obs_i^{[m,t]} = \obsalt_i^{[m,t]}$  be the largest common suffix of $\obs_i^{[\ell_t,t]} $ and $\obsalt_i^{[\ell_t,t]}$. Hence, $\obs_i^{[\ell_t,t]} \sim_b \obsaltalt_i^{[\ell_t,t]}$ for all $\obsaltalt_i^{[\ell_t,t]}$ such that $\obs_i^{[m,t]} = \obsaltalt_i^{[m,t]}$. In particular, we have $\obs_i^{[\ell_t,t]} \sim_b \obsalt_i^{[\ell_t,t]}$. By symmetry and transitivity of $\sim_b$ this implies $\obsalt_i^{[\ell_t,t]} \sim_b \obs_i^{[\ell_t,t]} \sim_b \obsaltalt_i^{[\ell_t,t]}$, so $\obsalt_i^{[\ell_t,t]} \sim_b \obsaltalt_i^{[\ell_t,t]}$. Since $\obs_i^{[m,t]} = \obsalt_i^{[m,t]}$  this implies $\obsalt_i^{[\ell_t,t]} \sim_b \obsaltalt_i^{[\ell_t,t]}$ for all $\obsaltalt_i^{[\ell_t,t]}$ with $\obsalt_i^{[m,t]} = \obsaltalt_i^{[m,t]}$. Hence,  $\obsalt_i^{[\ell_t,t]} \sim \obs_i^{[\ell_t,t]}$, showing symmetry. 

Finally, we show transitivity. Let  $\obs_i^{[\ell_t,t]}$ and $ \obsalt_i^{[\ell_t,t]}$ and $ \obsaltaltalt_i^{[\ell_t,t]}$ with $\obs_i^{[\ell_t,t]} \sim \obsalt_i^{[\ell_t,t]}$ and $\obsalt_i^{[\ell_t,t]} \sim \obsaltaltalt_i^{[\ell_t,t]}$ be given. 

Let  $\obs_i^{[m,t]} = \obsalt_i^{[m,t]}$  be the largest common suffix of $\obs_i^{[\ell_t,t]} $ and $\obsalt_i^{[\ell_t,t]}$ and let $\obsalt_i^{[m',t]} = \obsaltaltalt_i^{[m',t]}$ be the largest common suffix of $\obsalt_i^{[\ell_t,t]} $ and $\obsaltaltalt_i^{[\ell_t,t]}$. 
If $m \geq m'$, then $\obs_i^{[m,t]} = \obsalt_i^{[m,t]} = \obsaltaltalt_i^{[m,t]}$. We have $\obs_i^{[\ell_t,t]} \sim_b \obsaltalt_i^{[\ell_t,t]}$ for all $\obsaltalt_i^{[\ell_t,t]}$ such that $\obs_i^{[m,t]} = \obsaltalt_i^{[m,t]}$. Since $\obs_i^{[m,t]} = \obsaltaltalt_i^{[m,t]}$ is a common suffix of $\obs_i^{[\ell_t,t]}$ and $\obsaltaltalt_i^{[\ell_t,t]}$, this shows that $\obs_i^{[\ell_t,t]} \sim \obsaltaltalt_i^{[\ell_t,t]}$. 

If $m < m'$, then we can use the symmetry of $\sim$ to obtain $\obsaltaltalt_i^{[\ell_t,t]} \sim \obsalt_i^{[\ell_t,t]}$ and $\obsalt_i^{[\ell_t,t]} \sim \obs_i^{[\ell_t,t]}$. This reverses $m$ and $m'$, so now we can use the previous case to conclude that $\obsaltaltalt_i^{[\ell_t,t]} \sim \obs_i^{[\ell_t,t]}$. Again using symmetry, we get $\obs_i^{[\ell_t,t]} \sim \obsaltaltalt_i^{[\ell_t,t]}$. 
Hence, we conclude that transitivity holds and hence that $\sim$ is an equivalence relation.
\end{proof}

We also claimed that equivalence classes of $\sim$ are exactly the suffixes with some fixed suffix $\obs_i^{[m,t]}$. 

\begin{claim} \label{claim:suffix}
Let $c_i^t$ be a cluster, i.e.\ an equivalence class of $\sim$. Then there exists a suffix $\obs_i^{[m,t]}$ such that $c_i^t$ contains precisely the suffixes with suffix $\obs_i^{[m,t]}$.
\end{claim}

\begin{proof}
Consider the largest common suffix for each pair of suffixes $\obs_i^{[\ell_t,t]},\obsalt_i^{[\ell_t,t]}$ in $c_i^t$ and take a shortest $\obs_i^{[m,t]}$ among them. We prove by contraction that each suffix in $c_i^t$ has this suffix $\obs_i^{[m,t]}$. Suppose that there exists a suffix that does not have suffix $\obs_i^{[m,t]}$. Take this suffix and another suffix that does have suffix $\obs_i^{[m,t]}$. Then these two suffixes have a shorter largest common suffix than $\obs_i^{[m,t]}$, contradiction. Hence, all suffixes in  $c_i^t$ have suffix $\obs_i^{[m,t]}$.

Hence, it remains to prove that $c_i^t$ contains all suffixes with suffix $\obs_i^{[m,t]}$. Let $\obs_i^{[\ell_t,t]}$ and $\obsalt_i^{[\ell_t,t]}$ be two suffixes in $c_i^t$ with largest common suffix $\obs_i^{[m,t]}$. Then $\obs_i^{[\ell_t,t]} \sim_b \obsaltalt_i^{[\ell_t,t]}$ for all $\obsaltalt_i^{[\ell_t,t]}$ with  suffix $\obs_i^{[m,t]}$. If $\obsaltaltalt_i^{[\ell_t,t]}$ is any suffix with suffix $\obs_i^{[m,t]}$, then we hence have $\obs_i^{[\ell_t,t]} \sim_b \obsaltaltalt_i^{[\ell_t,t]}$. By symmetry and transitivity of $\sim_b$ this implies $\obsaltaltalt_i^{[\ell_t,t]} \sim_b \obs_i^{[\ell_t,t]} \sim_b \obsaltalt_i^{[\ell_t,t]}$, so $\obsaltaltalt_i^{[\ell_t,t]} \sim_b \obsaltalt_i^{[\ell_t,t]}$ for all $\obsaltalt_i^{[\ell_t,t]}$ such that $\obsaltaltalt_i^{[m,t]} = \obsaltalt_i^{[m,t]}$. Hence,  $\obsaltaltalt_i^{[\ell_t,t]} \sim \obs_i^{[\ell_t,t]}$, which shows that $\obsaltaltalt_i^{[\ell_t,t]} \in c_i^t$, so $c_i^t$  contains precisely the suffixes with suffix $\obs_i^{[m,t]}$.
\end{proof}

\vspace{30pt}

\subsubsection{Incrementality}

Next, we discuss Lemma 2. We first give some intuition why an additional condition is needed on top of belief-equivalence to obtain an incremental clustering. As shown before, if two suffixes of agent $i$ are belief-equivalent, then they induce the same belief over the pairs  $(s, \jointmultiagentbelief_{\neq i}^{h-t})$, which give the state and future policy of other agents. However, to show incrementality, we also need that the suffixes induce the same belief over agent $i$'s own future policy. This is trivial if we condition on the full suffix $\obs_i^{[\ell_t,t]}$, since we can compute agent $i$'s current cluster from $\obs_i^{[\ell_t,t]}$. However, it is no longer automatic if we condition on $\obs_i^{[v,t]}$ for some $v > \ell_t$. 

To solve this, the extra condition imposed by equivalence ensures that for `large' $v$ the two suffixes $\obs_i^{[v,t]}$ and $\obsalt_i^{[v,t]}$ are the same, while for `small' $v$ the missing observations are not required for determining agent $i$'s policy.

\begin{lemma}
Clustered sliding window memory is incremental.
\end{lemma}
\begin{proof}
The structure of the proof is similar to the proof of Lemma 1 of \citet{oliehoek2013incremental}. 

Fix an agent $i$ and a stage $t$. Let $\obs_i^{[\ell_t,t]}$ and $\obsalt_i^{[\ell_t,t]}$ with $\obs_i^{[\ell_t,t]} \sim \obsalt_i^{[\ell_t,t]}$ be given. To show incrementality, we have to show that when extending both suffixes with the same observation $\obs_i^{t+1} = \obsalt_i^{t+1}$, we have $\obs_i^{[\ell_{t+1},t+1]} \sim \obsalt_i^{[\ell_{t+1},t+1]}$.

Let $\obs_i^{[m,t+1]} = \obsalt_i^{[m,t+1]}$ be the largest common suffix of $\obs_i^{[\ell_{t+1},t+1]}$ and $\obsalt_i^{[\ell_{t+1},t+1]}$. Then $m\leq t{+}1$ since  $\obs_i^{t+1} = \obsalt_i^{t+1}$.  
Note that $\obs_i^{[\ell_t,t]}$ and $\obsalt_i^{[\ell_t,t]}$ have common suffix $\obs_i^{[m,t]}\!=\!\obsalt_i^{[m,t]}$. If $m = \ell_{t+1}$, then we have $\obs_i^{[\ell_{t+1},t+1]} = \obsalt_i^{[\ell_{t+1},t+1]}$ and hence trivially $\obs_i^{[\ell_{t+1},t+1]} \sim \obsalt_i^{[\ell_{t+1},t+1]}$. From now on assume that $m > \ell_{t+1}$. Then $\obs_i^{[m,t]} = \obsalt_i^{[m,t]}$ is the largest common suffix of $\obs_i^{[\ell_t,t]}$ and $\obsalt_i^{[\ell_t,t]}$. Since $\obs_i^{[\ell_t,t]} \sim \obsalt_i^{[\ell_t,t]}$, this implies 
\begin{align}\label{eq:lemma2result}
\obs_i^{[\ell_t,t]} \sim_b \obsaltalt_i^{[\ell_t,t]}\text{ for all }\obsaltalt_i^{[\ell_t,t]}\text{ with }\obs_i^{[m,t]} = \obsaltalt_i^{[m,t]} .
\end{align}

 Let $\jointact_{\neq i}^{t+1}$ denote the joint action of all agents except agent $i$. Let $v \in \{\ell_{t+1}, \ldots, t\}$. We first compute
\begin{align*}
&\Pr(s^{t+1}, \jointcluster_{\neq i}^{t}, \jointobs_{\neq i}^{t+1}, \obs_i^{t+1} \mid \obs_i^{[v,t]}) \\ &\quad= \sum_{s^t} \Pr(s^{t+1}, \jointcluster_{\neq i}^{t}, \jointobs_{\neq i}^{t+1}, \obs_i^{t+1}, s^t \mid \obs_i^{[v,t]}) 
\\ &\quad = \sum_{s^t}  \Pr(s^{t+1},\jointobs^{t+1} \mid  \jointcluster_{\neq i}^{t},\obs_i^{[v,t]}, s^t) \Pr(\jointcluster_{\neq i}^{t}, s^t \mid \obs_i^{[v,t]})  
\\ &\quad = \sum_{s^t, \jointact^{t+1}} \Big[ \Pr(s^{t+1},\jointobs^{t+1} \mid  \jointact^{t+1}, \jointcluster_{\neq i}^{t},\obs_i^{[v,t]}, s^t) \cdot \\[-10pt]
 & \qquad\qquad\qquad \Pr(\jointact^{t+1} \mid \jointcluster_{\neq i}^{t},\obs_i^{[v,t]}, s^t) \Pr(\jointcluster_{\neq i}^{t}, s^t \mid \obs_i^{[v,t]})  \Big]
\\ &\quad = \sum_{s^t, \jointact^{t+1}} \Big[ \Pr(s^{t+1},\jointobs^{t+1} \mid  \jointact^{t+1}, s^t) \Pr(\jointact_{\neq i}^{t+1} \mid \jointcluster_{\neq i}^{t}) \cdot \\[-10pt] 
& \qquad\qquad\qquad \Pr(\act_i^{t+1} \mid \jointcluster_{\neq i}^{t},\obs_i^{[v,t]}, s^t) \Pr(\jointcluster_{\neq i}^{t}, s^t \mid \obs_i^{[v,t]}) \Big].  
\end{align*}
In the final step, we remove the conditioning on $\jointcluster_{\neq i}^{t}$ and $\obs_i^{[v,t]}$ in $\Pr(s^{t+1},\jointobs^{t+1} \mid  \jointact^{t+1}, s^t)$ due to the Markov property of Dec-POMDPs and we remove the conditioning on $\obs_i^{[v,t]}$ and $s^t$ in $\Pr(\jointact_{\neq i}^{t+1} \mid \jointcluster_{\neq i}^{t})$ since $\jointcluster_{\neq i}^{t}$ already completely determines the actions the agents except agent $i$ take. 

To show that $\obs_i^{[\ell_{t+1},t+1]} \sim \obsalt_i^{[\ell_{t+1},t+1]}$, we have to show that  $\obs_i^{[\ell_{t+1},t+1]} \sim_b \obsaltalt_i^{[\ell_{t+1},t+1]}$ for all $\obsaltalt_i^{[\ell_{t+1},t+1]}$ such that $\obs_i^{[m,t+1]} = \obsaltalt_i^{[m,t+1]}$. Let $\obsaltalt_i^{[\ell_{t+1},t+1]}$ such that $\obs_i^{[m,t+1]} = \obsaltalt_i^{[m,t+1]}$ be given arbitarily.

Now we show that the final expression for the probability $\Pr(s^{t+1}, \jointcluster_{\neq i}^{t}, \jointobs_{\neq i}^{t+1}, \obs_i^{t+1} \mid \obs_i^{[v,t]})$ given above is the same when we use $\obsaltalt_i^{[v,t]}$ instead of $\obs_i^{[v,t]}$. We show this for each of the four factors individually:  
\begin{itemize} 
\item The first two factors do not depend on $\obs_i^{[v,t]}$. 
\item For the third factor we have to show that \[\Pr(\act_i^{t+1} \mid \jointcluster_{\neq i}^{t},\obs_i^{[v,t]}, s^t) = \Pr(\act_i^{t+1} \mid \jointcluster_{\neq i}^{t},\obsaltalt_i^{[v,t]}, s^t).\] We make a case distinction on $v$. If $v \geq m$, then $\obs_i^{[v,t]} = \obsaltalt_i^{[v,t]}$, so the expressions are trivially equal. If $v < m$, then  $\obs_i^{[v,t]} = \obsaltalt_i^{[v,t]}$ implies  $\obs_i^{[m,t]} = \obsaltalt_i^{[m,t]}$. Since all observation histories with suffix $\obs_i^{[m,t]}$ are clustered together (by Claim \ref{claim:suffix}), this means that $\obs_i^{[m,t]}$ determines the action $\act_i^{t+1}$ that agent $i$ takes. Hence, also in this case the factors are equal. 
\item
For the fourth factor, we have to show that \[\Pr(\jointcluster_{\neq i}^{t}, s^t \mid \obs_i^{[v,t]}) = \Pr(\jointcluster_{\neq i}^{t}, s^t \mid \obsaltalt_i^{[v,t]}).\] This holds since $\obs_i^{[\ell_t,t]} \sim_b \obsaltalt_i^{[\ell_t,t]}$ (see \eqref{eq:lemma2result}), which implies
\begin{equation*} \begin{split} \Pr(\forget(\jointcluster_{\neq i}^{t-1}, &\jointobs_{\neq i}^{t}), s^t \mid \obs_i^{[v,t]}) \\ &= \Pr(\forget(\jointcluster_{\neq i}^{t-1}, \jointobs_{\neq i}^{t}), s^t \mid \obsaltalt_i^{[v,t]}).\end{split}\end{equation*}
We use Claim \ref{cluster_claim1} to go from a distribution over pairs $(s^t, \forget(\jointcluster_{\neq i}^{t-1}, \jointobs_{\neq i}^{t}))$ to distribution over pairs $(s^t, \jointcluster_{\neq i}^{t})$.
\end{itemize}
This implies that \[\Pr\Big(s^t, \forget(\jointcluster_{\neq i}^{t-1}, \jointobs_{\neq i}^{t})\! ~\big|~  \!\obs_i^{[v,t]}\Big) = \Pr\Big(s^t, \forget(\jointcluster_{\neq i}^{t-1}, \jointobs_{\neq i}^{t})\!  ~\big|~  \!\obsaltalt_i^{[v,t]}\Big).\]

This then implies $\Pr(\jointcluster_{\neq i}^{t}, s^t \mid \obs_i^{[v,t]}) = \Pr(\jointcluster_{\neq i}^{t}, s^t \mid \obsaltalt_i^{[v,t]}) $ by aggregating over all clusters in $\forget(\jointcluster_{\neq i}^{t-1}, \jointobs_{\neq i}^{t})$ which end up in the same cluster $\jointcluster_{\neq i}^{t}$.

Finally, note that 
\begin{align*}
&\Pr(s^{t+1}, \jointcluster_{\neq i}^{t}, \jointobs_{\neq i}^{t+1} \mid \obs_i^{[v,t+1]}) \\&\quad= \frac{\Pr(s^{t+1}, \jointcluster_{\neq i}^{t}, \jointobs_{\neq i}^{t+1}, \obs_i^{t+1} \mid \obs_i^{[v,t]}) }{\Pr(\obs_i^{t+1} \mid \obs_i^{[v,t]}) } \\ &\quad= \frac{\Pr(s^{t+1}, \jointcluster_{\neq i}^{t}, \jointobs_{\neq i}^{t+1}, \obs_i^{t+1} \mid \obs_i^{[v,t]}) }{\sum_{s^{t+1}, \jointcluster_{\neq i}^{t}, \jointobs_{\neq i}^{t+1}}\Pr(s^{t+1}, \jointcluster_{\neq i}^{t}, \jointobs_{\neq i}^{t+1}, \obs_i^{t+1} \mid \obs_i^{[v,t]}) }.
\end{align*}
In this fraction, the numerator and denominator do not change when using $\obsaltalt_i^{[v,t]}$ instead of $\obs_i^{[v,t]}$ as shown above. So  \begin{align*} &\Pr(s^{t+1}, \jointcluster_{\neq i}^{t}, \jointobs_{\neq i}^{t+1} \mid \obs_i^{[v,t+1]}) \\ &\quad= \Pr(s^{t+1}, \jointcluster_{\neq i}^{t}, \jointobs_{\neq i}^{t+1} \mid \obsaltalt_i^{[v,t+1]}).\end{align*} 
We can now sum over all $\jointcluster_{\neq i}^{t}, \jointobs_{\neq i}^{t+1}$ that end up in the same cluster $\forget(\jointcluster_{\neq i}^{t}, \jointobs_{\neq i}^{t+1})$. 
Hence, we conclude that \begin{align*} &\Pr(s^{t+1}, \forget(\jointcluster_{\neq i}^{t}, \jointobs_{\neq i}^{t+1}) \mid \obs_i^{[v,t+1]}) \\ &\quad= \Pr(s^{t+1}, \forget(\jointcluster_{\neq i}^{t}, \jointobs_{\neq i}^{t+1}) \mid \obsaltalt_i^{[v,t+1]}) \end{align*} for  $v \in \{\ell_{t+1}, \ldots, t\}$. Since $\obs_i^{t+1} = \obsaltalt_i^{t+1}$, which holds since $m \leq t+1$,  it also holds for $v = t+1$.

This implies that  $\obs_i^{[\ell_{t+1},t+1]} \sim_b \obsaltalt_i^{[\ell_{t+1},t+1]}$ for all $\obsaltalt_i^{[\ell_{t+1},t+1]}$ such that $\obs_i^{[m,t+1]} = \obsaltalt_i^{[m,t+1]}$, which finally implies that  $\obs_i^{[\ell_{t+1},t+1]} \sim \obsalt_i^{[\ell_{t+1},t+1]}$. This means that identical extensions of equivalent clusters are equivalent, which is what we need to show that the clustering is incremental.
\end{proof}

\subsubsection{Probability-based clustering}

Finally, we show that approximate equivalence (used in probability-based clustering) is an equivalence relation. We denote approximate equivalence by $\approx$. By definition, we have $\obs_i^{[\ell_t,t]} \approx \obsalt_i^{[\ell_t,t]}$ iff $\obs_i^{[\ell_t,t]} \sim \obsalt_i^{[\ell_t,t]}$ or $\Pr(\obs_i^{[m,t]}) \leq p_{\max}$, where $\obs_i^{[m,t]}$ is the largest common suffix of $\obs_i^{[\ell_t,t]}$ and $\obsalt_i^{[\ell_t,t]}$ 

\begin{claim} $\approx$ is an equivalence relation.
\end{claim} 
\begin{proof}
Fix $p_{\max}$. 

Reflexivity of $\approx$ follows directly from the reflexivity of $\sim$. 

We now show that $\approx$ is symmetric. Let $\obs_i^{[\ell_t,t]}$ and $ \obsalt_i^{[\ell_t,t]}$ with $\obs_i^{[\ell_t,t]} \approx \obsalt_i^{[\ell_t,t]}$ be given. Let $\obs_i^{[m,t]} = \obsalt_i^{[m,t]}$ be their largest common suffix. Then $\obs_i^{[\ell_t,t]} \sim \obsalt_i^{[\ell_t,t]}$ or $\Pr(\obs_i^{[m,t]}) \leq p_{\max}$. By symmetry of $\sim$ this implies $\obsalt_i^{[\ell_t,t]} \sim \obs_i^{[\ell_t,t]}$ or $\Pr(\obsalt_i^{[m,t]}) \leq p_{\max}$, so $\obsalt_i^{[\ell_t,t]} \approx \obs_i^{[\ell_t,t]}$, showing symmetry.

We now show that $\approx$ is transitive.  Let  $\obs_i^{[\ell_t,t]}$, $ \obsalt_i^{[\ell_t,t]}$ and $ \obsaltaltalt_i^{[\ell_t,t]}$ with $\obs_i^{[\ell_t,t]} \approx \obsalt_i^{[\ell_t,t]}$ and $\obsalt_i^{[\ell_t,t]} \approx \obsaltaltalt_i^{[\ell_t,t]}$ be given. 

Let  $\obs_i^{[m,t]} = \obsalt_i^{[m,t]}$  be the largest common suffix of $\obs_i^{[\ell_t,t]} $ and $\obsalt_i^{[\ell_t,t]}$ and let $\obsalt_i^{[m',t]} = \obsaltaltalt_i^{[m',t]}$ be the largest common suffix of $\obsalt_i^{[\ell_t,t]} $ and $\obsaltaltalt_i^{[\ell_t,t]}$. Since $\approx$ is symmetric we can assume without loss of generality that $m \geq m'$, as in the proof of Claim \ref{claim:eqrel}. Then we have $\obs_i^{[m,t]} = \obsalt_i^{[m,t]} = \obsaltaltalt_i^{[m,t]}$, so if $\obs_i^{[m'',t]} = \obsaltaltalt_i^{[m'',t]}$ is the largest common suffix of $\obs_i^{[\ell_t,t]} $ and $\obsaltaltalt_i^{[\ell_t,t]}$, then $m'' \leq m$. If $\Pr(\obs_i^{[m,t]}) \leq p_{\max}$, then also $\Pr(\obs_i^{[m'',t]}) \leq p_{\max}$, so $\obs_i^{[\ell_t,t]} \approx \obsaltaltalt_i^{[\ell_t,t]}$. Otherwise, we have  $\obs_i^{[\ell_t,t]} \sim \obsalt_i^{[\ell_t,t]}$. Hence, the cluster containing $\obs_i^{[\ell_t,t]}$ contains all suffixes with suffix $\obs_i^{[m,t]}$ by Claim \ref{claim:suffix}. Since $\obs_i^{[m,t]} = \obsaltaltalt_i^{[m,t]}$, it then also contains $\obsaltaltalt_i^{[m,t]}$, so $\obs_i^{[m,t]} \sim \obsaltaltalt_i^{[m,t]}$ and hence $\obs_i^{[m,t]} \approx \obsaltaltalt_i^{[m,t]}$. Hence, we conclude that transitivity holds and hence that $\approx$ is an equivalence relation.
\end{proof}

\subsection{Clustering: Empirical Evaluation}

\begin{table}
\centering
\begin{tabular}{@{}r|r|rrr@{}} \hline
 & & & \multitwo{only possible} & \\
benchmark & $|\observations_i|$ & $|\observations_i|^3$ & & clustered \\ \hline
\dectiger & 2 & 8 & 8 & 8 \\
\grid & 2 & 8 & 8 & 8 \\
\boxpush & 5 & 125 & 30 & 30 \\
\ff & 2 & 8 & 8 & 8 \\
\gridthree & 9 & 729 & 111 & 9 \\
\mars & 8 & 512 & 38 & 34 \\
\hotel & 4 & 64 & 8 & 4 \\
\recycling & 2 & 8 & 5 & 2 \\
\broadcast & 2 & 8 & 8 & 1 \\ \hline
\end{tabular}
\caption{Maximum number of clusters for a single stage/agent, for sliding window memory with $k=3$ for a good policy with $\horizon=10$.} \label{tab:cluster}
\end{table}

Table \ref{tab:cluster} provides a short empirical evaluation of (clustered) sliding window memory. We consider a window size of $k=3$ observations. For each benchmark, we compute a good policy for horizon $h=10$ using the $Q_{\mathsf{MDP},2}$ heuristic and $L = 10\,000$ using clustered sliding window memory, and report the maximum number of clusters for a single stage and a single agent for that policy. For comparison, we also give the number of clusters without any clustering, which is $|\observations_i|^k$, and the maximum number of clusters for a policy computed using what we call \emph{only-possible} clustering, which does not assign a cluster (equivalently, assigns an arbitrary cluster) to suffixes which have probability 0, using the same heuristic and limit $L$. We note that Table \ref{tab:cluster} only gives the number of clusters for near-optimal policies, while the number of clusters in sub-optimal policies can be different. We now discuss the results, split over three cases: (1) a case where clustering helps a lot, (2) a case where clustering does not help, and (3) a case where clustering is somewhat helpful.

\paragraph{Case 1: Substantially fewer clusters for all (good) policies.} On the benchmarks where lossless clustering \cite{oliehoek2009lossless} performs well, namely \gridthree{}, \hotel{}, \recycling{}, and \broadcast{}, clustered sliding window memory is also able to substantially reduce the number of clusters, despite the conditions for clustered sliding window memory being stricter than those for lossless clustering. In contrast, the only-possible clustering yields (much) smaller reduction.  Running times are considerably lower when using clustered sliding window memory: a factor 4 speedup for \hotel{} and \recycling{}, a factor 8 speedup for \gridthree{} and a factor 50 speedup for \broadcast{}, compared to only-possible clustering. Moreover, clustered sliding window memory finds a better value for \gridthree{}: 4.685 instead of 4.678.

\paragraph{Case 2: Low overhead when clustering is not possible.} On \dectiger{} and \boxpush{},  the algorithm returns the same policy when using only-possible clustering as when using clustered sliding window memory. For \boxpush{}, using only-possible clustering does have a substantial effect. Due to the overhead of clustering, the running time is around 50\% larger on \dectiger{} and 10\% larger on \boxpush{} than for only-possible clustering.  For \dectiger{}, lossless clustering is able to cluster some observation histories where the order of the observations does not matter. This is not the case with sliding window memory as the order determines when the observation leaves the window. On \ff{}, clustering does help somewhat (but not for all stages/agents), and the benefit approximately offsets the overhead.

\paragraph{Case 3: Smaller search space leading to better policies.} On \grid{} and \mars{}, clustered sliding window memory helps substantially to reduce the search space compared to only-possible clustering. For \grid{}, clustering is mainly possible for sub-optimal policies, and some stages in near-optimal policies. Since the number of policies is exponential in the number of clusters, even a small reduction can have an effect. Using clustered sliding window memory provides better policies than using only-possible clustering: an improvement from 6.811 to 6.821 on \grid{} and from 26.307 to 26.310 on \mars{}. However, for \mars{} this does come at the cost of a running time that is 80\% larger. This increase is primarily because more heuristics need to be computed due to the larger exploration of the search space, and only to a lesser extend due to the time required for the clustering itself. On \grid, the running time is only 15\% larger.

\section{\ourLB: Pseudocode}\label{appB}

Algorithm \ref{alg:rsmaa} presents pseudocode for \ourLB (Sec.~\ref{sec:loose_heuristics}), showing in particular where queue pruning (App.~\ref{appC}) and horizon reduction (App.~\ref{appD}) are done. We do not show terminating heuristic computations after $M$ steps. 

\begin{algorithm}[H]
    \caption{\ourLB}
    \label{alg:rsmaa}
    
    \begin{algorithmic}[1]
    \Require{$\horizon$: horizon, $\jointppol$: initial policy (can be nonempty for computing heuristics), $k$: window size, $L$: limit used in $\progress$, $r$: maximum horizon after horizon reduction, $\textit{mode}$: `heuristic' or `policy-finding'}
            \Function{\ourLB}{$\horizon,\jointppol, L, r, \textit{mode}$}
            \State $t \gets 0$
            \If{\textit{mode} = \textit{heuristic}} \Comment{horizon reduction}
                \State $t \gets \max\{h-r, 0\}$; $h \gets h-t$
            \EndIf
            \State $q\gets \textsc{PriorityQueue}()$\Comment{sorted in descending order}
            \State $q.\text{push}(\langle\min(\jointppol.\text{heuristics}), \jointppol\rangle)$; $N \gets 0$
            \While{true}
                \State $v, \jointppol \gets q.\text{pop}()$
                \If{$\jointppol$ is fully specified}
                     \textbf{return} $v, \jointppol$ \EndIf
                \If{$\progress(\jointppol) < N$}
                     \textbf{continue} \Comment{queue pruning}  \EndIf  
                \State $N \gets N+1$ 
                \State $\obstrace \gets$ largest LOH for which $\jointppol(\obstrace) \neq \bot$
                \If{$\obstrace$ is largest LOH of some length}
                \State $\jointppol.\text{cluster\_policy}(\text{window} = k)$ 
                \EndIf
                \State $\obstrace \gets \text{next}(\obstrace)$; $a \gets \text{agent}(\obstrace)$
                \If{$a = n$ and $\stage{\jointppol} = h-1$}  \Comment{last stage}
                    \State $\jointpolicy \gets \jointppol$ 
                    \ForAll{LOHs $\obstrace'$ of agent $n$ of length $h-1$} 
                        \For{$\textit{act} \in \actions_a$}
                            \State compute reward $R[\textit{act}]$ in stage $h{-}1$ plus horizon $t$ terminal reward  for action $\textit{act}$ for LOH $\obstrace'$
                        \EndFor
                        \State select $\jointpolicy(\obstrace')$ from $\argmax_{\textit{act} \in \actions_a} R[\textit{act}]$
                    \EndFor
                    \State $q.\text{push}(\langle \jointpolicy.\text{evaluate\_policy}(), \jointpolicy\rangle)$
                \Else \Comment{extend policy and compute new heuristics}
                    \For{$\textit{act} \in \actions_a$}
                        \State $\jointppol' \gets \jointppol.\text{copy()}$; $\jointppol'(\obstrace) \gets \textit{act}$
                        \State update $\jointppol'.\text{heuristics}$
                        \State $q.\text{push}(\langle\min(\jointppol'.\text{heuristics}), \jointppol'\rangle)$
                    \EndFor
                \EndIf
            \EndWhile
        \EndFunction
    \end{algorithmic}
\end{algorithm}

\section{\ourLB: Queue pruning}\label{appC}

In this section, we give further intuition for the queue pruning in \ourLB.  We first briefly recap how queue pruning works. Then we explain the necessity of queue pruning, and provide intuition for the progress measure $\progress$. In addition, we prove that the progress measure ensures that the queue always contains a policy which is not pruned, and hence that \ourLB always finds a fully specified policy. 

\paragraph{Queue pruning.} We prune the \astar priority queue using a progress measure $\progress$. Partial policies in the priority queue are still processed in descending order of heuristic value. However, a partial policy $\jointppol$ is only expanded if $\progress(\jointppol) \geq N$, where $N$ is the number of policies expanded already. Otherwise, the partial policy $\jointppol$ is pruned, i.e.\ removed from the queue without further processing.

\paragraph{Necessity of queue pruning.} When running  (small-step)  \maastar to find an optimal policy, it is not possible to prune the priority queue. It can be observed empirically that too many nodes are expanded even before the first fully specified policy is added to the priority queue, to be able to find a policy in reasonable time. 
Hence, we have to limit the number of policies expanded in the shallow levels of the search tree.
In fact, we really have to limit the number of policies expanded up to each level of the tree, to ensure that the algorithm will explore deeper policies sufficiently fast.

\subsection{Intuition for $\progress$}

We first recall the two design goals for $\progress$:
\begin{itemize}
\item It has to limit the number of policies expanded up to each level of the search tree, to ensure that a fully specified policy is eventually found.
\item It should prune the least promising policies.
\end{itemize}

The first design goal is met by using a progress measure that increases by at least 1 each time we go down a level in the tree. We now focus on the intuition why our progress measure meets the second design goal.

When using an admissible heuristic for policy finding, the heuristic values typically go down as we go deeper in the search tree. Namely, a larger part of the heuristic value becomes the actual realized reward in the stages for which the policy is already specified, rather than the overestimation provided by the heuristic for remaining stages (for which the policy is not specified yet). Hence, for admissible heuristics, a deep policy with some heuristic value is more likely to have a good policy as descendant than a shallow policy with the same heuristic value. Hence, it is preferable to prune the shallow node when aiming to find a good policy.

\paragraph{First attempt: progress measure $\progress'$.} Due to differences in the number of clusters per stage between different policies, we should however not directly use the depth of the node in the search tree (which is just the total number of LOHs to which an action is assigned), but rather take into account how many stages this represents, and within each stage, for how many agents the policy has been fully specified. A first idea for a progress measure could therefore be as follows: Let $\jointppol$ be a partial policy such that $\jointppol$ has specified an action for all clusters of length $\stage{\jointppol}-1$, for all clusters of length $\stage{\jointppol}$ for $i$ out of $n$ agents, and for $c$ out of $|C_{i+1,\stage{\jointppol}}|$ clusters of length $\stage{\jointppol}$ of agent $i+1$. Then we could consider \[
\progress'(\jointppol) = \stage{\jointppol} \cdot L + i \cdot \frac{L}{n} + c \cdot \frac{L}{n \left|C_{i+1,\stage{\jointppol}}\right|}.
\]
The progress measure $\progress'$ aims to divide the available iterations evenly over the stages, within each stage evenly over the agents, and per agent evenly over the different clusters of the agent in that stage.

\newpage

\paragraph{Progress measure $\progress$.} The progress measure $\progress'$ gives comparable, but slightly worse results than the progress measure $\progress$. With $\progress$, we also take into account the probability of each of the clusters for the agent $i+1$ for which we currently specify actions.  We recall that $\progress$ is defined as
\[
\progress(\jointppol) = \stage{\jointppol} \cdot L + i \cdot \frac{L}{n} + c + p \cdot\left(\frac{L}{n} - \left|C_{i+1,\stage{\jointppol}}\right| \right).
\]
where $p$ is  the probability that the LOH of agent $i+1$ is in one of the first $c$ clusters, i.e.\ the probability of the clusters of agent $i+1$ to which $\jointppol$ already assigns an action.

To provide intuition for the progress measure $\progress$, we note the following. The contribution of the overestimation of the heuristic on a given cluster to the overall heuristic value, is proportional to the probability of that cluster. 
Hence,  if so far we mostly assigned actions to clusters with small probability (for the current agent), the heuristic value will be more optimistic, because the overestimations that we reduced by assigning actions have a smaller weight. Hence, a smaller $p$ should correspond to a smaller progress measure, so that it is pruned rather than a partial policy with the same heuristic value progressed to the same stage and agent, but which has assigned actions to clusters with a larger probability.

However, when completely removing $c$ from the progress measure and replacing it by $p \cdot \frac{L}{n}$, we no longer ensure that the queue always contains a policy which is not pruned. Therefore, we also include $c$ in the progress measure (which is at most $\left|C_{i+1,\stage{\jointppol}}\right|$), and include the probability $p$ with the remaining weight $\left(\frac{L}{n} - \left|C_{i+1,\stage{\jointppol}}\right| \right)$. This  ensures that progress measure always increases by 1 when going down a level in the tree, and hence that the queue always contains a policy which is not pruned. We prove this next.

\subsection{Proofs}

The definition of the progress measure has been chosen to ensure that $\progress(\jointppol') \geq \progress(\jointppol)+1$ if $\jointppol'$ is a child of $\jointppol$ in the search tree. Hence, if $\jointppol$ is expanded, then its child $\jointppol'$ is a possible candidate for expansion in the next iteration. We formalize this in two claims and a corollary.

\begin{claim}\label{claim:prog_increases}
 Assume that $L \geq n |C_{i, t}|$ for all $0 \leq t \leq \horizon-1$ and all $i \in \agents$.  Let $\jointppol$ be a partial policy and let $\jointppol'$ be a child of $\jointppol$ in the small-step search tree. Then $\progress(\jointppol') \geq \progress(\jointppol)+1$. 
\end{claim}
\begin{proof}
Assume that $\jointppol$ has specified an action for all clusters of length $\stage{\jointppol}-1$, for $i$ out of $n$ agents, and for $c$ out of $|C_{i+1,\stage{\jointppol}}|$ clusters of agent $i+1$, and let $p$ be the probability that the LOH of agent $i+1$ is in one of the first $c$ clusters. Define $\stage{\jointppol'}, i', c'$, and $p'$ similarly for $\jointppol'$.

If $c < |C_{i+1,\stage{\jointppol}}| - 1$, then $\jointppol'$ just specifies an action for an additional cluster for agent $i+1$, so $\stage{\jointppol'} = \stage{\jointppol}$, $i' = i$, $c' = c + 1$ and $p' \geq p$. Then 
\begin{align*}
&\progress(\jointppol') = \stage{\jointppol'} \!\cdot\! L + i' \!\cdot\!  \frac{L}{n} + c' + p'\left(\!\frac{L}{n} \!-\! \left|C_{i'+1,\stage{\jointppol'}}\right|\! \right) \\
&\quad\geq \stage{\jointppol} \cdot L + i \cdot \frac{L}{n} + (c+1) + p\left(\frac{L}{n} - \left|C_{i+1,\stage{\jointppol}}\right| \right) \\
&\quad = \progress(\jointppol) + 1.
\end{align*}

If $c = |C_{i+1,\stage{\jointppol}}| - 1$, then $\jointppol'$ specifies the last action for agent $i+1$. If $i+1 < n$, then $\stage{\jointppol'} = \stage{\jointppol}$, $i' = i+1$, $c' = 0$ and $p' = 0$. Then 
\begin{align*}
&\progress(\jointppol') = \stage{\jointppol'} \cdot L + i' \cdot  \frac{L}{n} = \stage{\jointppol} \cdot L + i \cdot \frac{L}{n} + \frac{L}{n} \\\
&\quad = \stage{\jointppol}  \!\cdot\! L + i  \!\cdot\! \frac{L}{n} + \left|C_{i+1,\stage{\jointppol}}\right| + \left(\!\frac{L}{n} \!-\! \left|C_{i+1,\stage{\jointppol}}\right|\! \right) \\
&\quad \geq \stage{\jointppol}  \!\cdot\! L + i  \!\cdot\! \frac{L}{n} + (c+1) + p\left(\!\frac{L}{n} \!-\! \left|C_{i+1,\stage{\jointppol}}\right|\! \right) \\
&\quad = \progress(\jointppol) + 1. 
\end{align*}

Finally, if $c = |C_{i+1,\stage{\jointppol}}| - 1$ and $i+1 = n$, then $\stage{\jointppol'} = \stage{\jointppol}+1$, $i' = 0$, $c' = 0$ and $p' = 0$. Then 
\begin{align*}
&\progress(\jointppol') = \stage{\jointppol'} \cdot L = \stage{\jointppol} \cdot L + (n\!-\!1) \!\cdot\! \frac{L}{n} + \frac{L}{n} \\
&\quad = \stage{\jointppol}  \!\cdot\! L + i \!\cdot\! \frac{L}{n} + \left|C_{i+1,\stage{\jointppol}}\right| + \left(\!\frac{L}{n} \!-\! \left|C_{i+1,\stage{\jointppol}}\right|\! \right) \\
&\quad \geq \stage{\jointppol}  \!\cdot\! L + i  \!\cdot\! \frac{L}{n} + (c+1) + p\left(\!\frac{L}{n} \!-\! \left|C_{i+1,\stage{\jointppol}}\right|\! \right) \\
&\quad = \progress(\jointppol) + 1.
\end{align*}

In each case, we use that $L \geq n |C_{i, t}|$ when using that the weight for the probability,  $\left(\frac{L}{n} \!-\! \left|C_{i+1,\stage{\jointppol}}\right| \right)$, is nonnegative.

Hence, we always have  $\progress(\jointppol') \geq \progress(\jointppol)+1$. 
\end{proof}

This claim implies that it is an invariant of the algorithm that the queue always contains a policy which is not pruned.

\begin{claim}\label{claim:invariant}
 Assume that $L \geq n |C_{i, t}|$ for all $0 \leq t \leq \horizon-1$ and all $i \in \agents$. Until a fully specified policy is found, the queue always contains a policy $\jointppol$ such that $\progress(\jointppol) \geq N$, where $N$ is the number of policies expanded already.
\end{claim}
\begin{proof}
Initially, we have $N=0$ and the queue starts with a policy $\jointppol$ with $\progress(\jointppol) \geq 0$.\footnote{For heuristic computations, the initial policy $\jointppol$ in the queue is not necessarily the empty policy.} Now suppose that we expand some policy $\jointppol$ when $N$ policies are expanded already. Then $\progress(\jointppol) \geq N$. If $\jointppol$ is not fully specified, then we add its children $\jointppol'$ to the queue which satisfy  \[\progress(\jointppol') \geq \progress(\jointppol)+1 \geq N+1\] by Claim \ref{claim:prog_increases}. Since $N+1$ policies are expanded after expanding  $\jointppol$, this completes the proof.
\end{proof}

\begin{corollary}\label{cor:prune}
Assume that $L \geq n |C_{i, t}|$ for all $0 \leq t \leq \horizon-1$ and all $i \in \agents$. Then
\ourLB finds a fully specified policy within $\horizon \cdot L$ policy expansions.
\end{corollary}
\begin{proof}
For any partial policy $\jointppol$, we have $\progress(\jointppol) \leq h \cdot L$, with equality if and only if $\jointppol$ is fully specified. Now suppose that we have already expanded $\horizon \cdot L$ policies, but have not yet found a fully specified policy. Then Claim \ref{claim:invariant} implies that the queue contains a policy with  $\progress(\jointppol) \leq h \cdot L$, which is hence a fully specified policy. Moreover, we prune all policies which are not fully specified policy, so after $\horizon \cdot L$ policy expansions \ourLB returns the fully specified policy with the highest value in the queue if it has not returned a fully specified policy earlier already.
\end{proof}

\section{Horizon reduction}\label{appD}
  
We now give further explanation for the concept of horizon reduction. For both \ourLB and \ourUB, we have a \emph{policy-finding mode}, a \emph{horizon reduction step} and a \emph{heuristic (computation) mode}. In the initial call, both algorithms are in the policy-finding mode. In all further (recursive) calls, \ourLB and \ourUB make a horizon reduction step when called with a horizon $h' > r$ and are in heuristic mode when called with a horizon $h' \leq r$. 
In policy-finding mode, no horizon reduction is applied, since this precludes finding a policy for the full horizon $h$.\footnote{For \ourUB, it would be possible to start in horizon reduction mode. Then \ourUB terminates with a policy for a horizon $r$ Dec-POMDP with terminal rewards, and the corresponding policy value is an upper bound. Initial experiments indicate that this is typically faster, but gives worse bounds.} 
In a horizon reduction step, we reduce computing a heuristic for a horizon $\horizon'$ Dec-POMDP to computing a heuristic for a horizon $r$ Dec-POMDP with terminal rewards representing the remaining $\horizon'-r$ stages (as explained in Section~\ref{sec:loose_heuristics} and \ref{sec:terminal_heuristic} of the main paper). In heuristic mode, we do not have to find a policy, but only return an upper bound for the value of an optimal policy.

In policy-finding and heuristic mode, \ourLB and \ourUB behave similarly to \rsmaa \cite{DBLP:conf/ijcai/Koops0JS23}. 
In particular, we compute a heuristic corresponding to revealing the $d$th cluster for each agent, which in practice means that we consider all possible joint clusters $\jointcluster \in \bigtimes_{i \in \agents} C_{i,d}$ at stage $d$. 
For each cluster $\jointcluster$, we consider the  shortened partial policy $\jointppol|_{\jointcluster}$, which gives the actions the agents take when receiving particular observations from stage $d$ onwards. 
Then we compute heuristics for horizon $h' - d$ with initial policy $\jointppol|_{\jointcluster}$, and weigh these with the probabilities of the clusters $\jointcluster$ to compute a heuristic. In both modes, $d$ starts at 1. 
The difference between these modes is as follows.
In heuristic mode, $d$ increases to $d'$ when using the $Q_{d'}$ heuristic. In policy-finding mode for \ourLB, $d$ increases further as the stage of the partial policies $\jointppol$ in the queue increase.

\begin{figure}[t]
\centering
\begin{tikzpicture}[x=0.75pt,y=0.75pt,yscale=-0.9,xscale=1]

\draw   (90,11) -- (160,11) -- (160,60.5) -- (90,60.5) -- cycle ;
\draw    (124,60) -- (124.15,77.5) ;
\draw [shift={(124.17,79.5)}, rotate = 269.51] [color={rgb, 255:red, 0; green, 0; blue, 0 }  ][line width=0.75]    (10.93,-3.29) .. controls (6.95,-1.4) and (3.31,-0.3) .. (0,0) .. controls (3.31,0.3) and (6.95,1.4) .. (10.93,3.29)   ;
\draw   (90,80) -- (160,80) -- (160,129.5) -- (90,129.5) -- cycle ;
\draw    (125,129) -- (125.16,158.5) ;
\draw [shift={(125.17,160.5)}, rotate = 269.7] [color={rgb, 255:red, 0; green, 0; blue, 0 }  ][line width=0.75]    (10.93,-3.29) .. controls (6.95,-1.4) and (3.31,-0.3) .. (0,0) .. controls (3.31,0.3) and (6.95,1.4) .. (10.93,3.29)   ;
\draw   (90,160) -- (160,160) -- (160,209.5) -- (90,209.5) -- cycle ;
\draw    (160.5,173) .. controls (178,170) and (178,145)  .. (178,130) ;
\draw    (178,130) .. controls (180, 95) and (215,95)  .. (215.14,139.05) ;
\draw [shift={(215.17,140.67)}, rotate = 270] [color={rgb, 255:red, 0; green, 0; blue, 0 }  ][line width=0.75]    (10.93,-3.29) .. controls (6.95,-1.4) and (3.31,-0.3) .. (0,0) .. controls (3.31,0.3) and (6.95,1.4) .. (10.93,3.29)   ;
\draw   (180,140) -- (250,140) -- (250,171.67) -- (180,171.67) -- cycle ;
\draw  [dash pattern={on 0.84pt off 2.51pt}]  (215.14,172.02) .. controls (214.75,196.94) and (206.37,195.65) .. (161.4,194.21) ;
\draw [shift={(160.05,194.17)}, rotate = 1.8] [color={rgb, 255:red, 0; green, 0; blue, 0 }  ][line width=0.75]    (10.93,-3.29) .. controls (6.95,-1.4) and (3.31,-0.3) .. (0,0) .. controls (3.31,0.3) and (6.95,1.4) .. (10.93,3.29)   ;

\draw (125, 20) node [anchor=center][inner sep=0.75pt]   [align=center] {\ourLB};
\draw (125, 37) node [anchor=center][inner sep=0.75pt]   [align=center] {{\small policy-finding}};
\draw (125, 50) node [anchor=center][inner sep=0.75pt]   [align=center] {$\horizon = 100$};
\draw (125, 89) node [anchor=center][inner sep=0.75pt]   [align=center] {\ourLB};
\draw (125, 104) node [anchor=center][inner sep=0.75pt]   [align=center] {{\small hor. reduction}};
\draw (125, 119) node [anchor=center][inner sep=0.75pt]   [align=center] {$\horizon = 99$};
\draw (125, 169) node [anchor=center][inner sep=0.75pt]   [align=center] {\ourLB};
\draw (125, 184) node [anchor=center][inner sep=0.75pt]   [align=center] {{\small heuristic}};
\draw (125, 199) node [anchor=center][inner sep=0.75pt]   [align=center] {$\horizon = 3$};
\draw (215, 149) node [anchor=center][inner sep=0.75pt]   [align=center] {MDP};
\draw (215, 163) node [anchor=center][inner sep=0.75pt]   [align=center] {$\horizon = 96$};
\draw (128, 138) node [anchor=north west][inner sep=0.75pt]   [align=left] {{\small horizon}};
\draw (128, 148) node [anchor=north west][inner sep=0.75pt]   [align=left] {{\small reduction}};
\draw (170, 173) node [anchor=north west][inner sep=0.75pt]   [align=left] {{\small terminal}};
\draw (170, 182) node [anchor=north west][inner sep=0.75pt]   [align=left] {{\small reward}};
\end{tikzpicture}
\caption{Schematic call graph for \ourLB, for $r=3$. A dotted line indicates that a result is provided.}\label{fig:callgraph}
\end{figure}
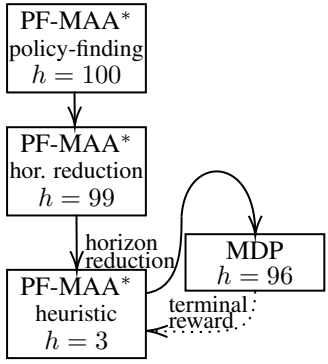

\paragraph{Call graph for \ourLB.} Figure \ref{fig:callgraph} shows a schematic call graph for \ourLB (for $r=3$). Initially, we have $d=1$ and call \ourLB in horizon reduction mode with horizon 99. Then the horizon is reduced to 3, and the corresponding call to \ourLB in heuristic mode uses the terminal rewards, which are values of an MDP with horizon 96, when needed.  In the figure, we only show these initial calls. Later in the execution of the algorithm, $d$ increases; then the initial call of \ourLB in policy-finding mode makes calls to \ourLB with horizon $100-d$, which will in turn make a horizon reduction if $100-d \geq r$, and use values of an MDP with horizon $100-d-r$ as terminal rewards.

\paragraph{Call graph for \ourUB.}  Figure \ref{fig:callgraphUB} shows a schematic call graph for \ourUB (for $r=5$). Initially, we have $d=1$ and call \ourUB in horizon reduction mode with horizon 99. Then the horizon is reduced to 5, and the corresponding call to \ourUB in heuristic mode uses the terminal rewards, for which it makes calls to \ourUB with horizon 94. This again makes a horizon reduction, and the corresponding call to \ourUB in heuristic mode with horizon 5 uses terminal rewards with horizon 89, which in turn uses  terminal rewards with horizon 84, etc., until we reach horizon 4 for which we directly compute a heuristic. As for \ourLB, the figure only shows the case where $d=1$, while $d$ can increase during the  execution of the algorithm.

\begin{figure}[t]
\centering
\begin{tikzpicture}[x=0.75pt,y=0.75pt,yscale=-0.9,xscale=1]

\draw   (179.82,20) -- (249.82,20) -- (249.82,69.15) -- (179.82,69.15) -- cycle ;
\draw   (179.82,90) -- (249.82,90) -- (249.82,139.15) -- (179.82,139.15) -- cycle ;
\draw   (179.82,160) -- (249.82,160) -- (249.82,209.15) -- (179.82,209.15) -- cycle ;
\draw   (279.82,90) -- (349.82,90) -- (349.82,139.15) -- (279.82,139.15) -- cycle ;
\draw   (279.82,160) -- (349.82,160) -- (349.82,209.15) -- (279.82,209.15) -- cycle ;
\draw   (380.82,90) -- (450.82,90) -- (450.82,139.15) -- (380.82,139.15) -- cycle ;
\draw   (380.82,160) -- (450.82,160) -- (450.82,209.15) -- (380.82,209.15) -- cycle ;
\draw    (214.82,69) -- (214.91,88.15) ;
\draw [shift={(214.92,90.15)}, rotate = 269.73] [color={rgb, 255:red, 0; green, 0; blue, 0 }  ][line width=0.75]    (10.93,-3.29) .. controls (6.95,-1.4) and (3.31,-0.3) .. (0,0) .. controls (3.31,0.3) and (6.95,1.4) .. (10.93,3.29)   ;
\draw    (214.82,139) -- (214.91,158.15) ;
\draw [shift={(214.92,160.15)}, rotate = 269.73] [color={rgb, 255:red, 0; green, 0; blue, 0 }  ][line width=0.75]    (10.93,-3.29) .. controls (6.95,-1.4) and (3.31,-0.3) .. (0,0) .. controls (3.31,0.3) and (6.95,1.4) .. (10.93,3.29)   ;
\draw    (314.82,139) -- (314.91,158.15) ;
\draw [shift={(314.92,160.15)}, rotate = 269.73] [color={rgb, 255:red, 0; green, 0; blue, 0 }  ][line width=0.75]    (10.93,-3.29) .. controls (6.95,-1.4) and (3.31,-0.3) .. (0,0) .. controls (3.31,0.3) and (6.95,1.4) .. (10.93,3.29)   ;
\draw    (414.82,139) -- (414.91,158.15) ;
\draw [shift={(414.92,160.15)}, rotate = 269.73] [color={rgb, 255:red, 0; green, 0; blue, 0 }  ][line width=0.75]    (10.93,-3.29) .. controls (6.95,-1.4) and (3.31,-0.3) .. (0,0) .. controls (3.31,0.3) and (6.95,1.4) .. (10.93,3.29)   ;
\draw  [dash pattern={on 0.84pt off 2.51pt}]  (279.63,190.13) -- (251.63,190.13) ;
\draw [shift={(249.63,190.13)}, rotate = 360] [color={rgb, 255:red, 0; green, 0; blue, 0 }  ][line width=0.75]    (10.93,-3.29) .. controls (6.95,-1.4) and (3.31,-0.3) .. (0,0) .. controls (3.31,0.3) and (6.95,1.4) .. (10.93,3.29)   ;
\draw  [dash pattern={on 0.84pt off 2.51pt}]  (379.63,190.13) -- (351.63,190.13) ;
\draw [shift={(349.63,190.13)}, rotate = 360] [color={rgb, 255:red, 0; green, 0; blue, 0 }  ][line width=0.75]    (10.93,-3.29) .. controls (6.95,-1.4) and (3.31,-0.3) .. (0,0) .. controls (3.31,0.3) and (6.95,1.4) .. (10.93,3.29)   ;
\draw    (269.9,90.22) .. controls (286.48,10.62) and (313.27,64.81) .. (314.6,87.96) ;
\draw [shift={(314.65,89.68)}, rotate = 270] [color={rgb, 255:red, 0; green, 0; blue, 0 }  ][line width=0.75]    (10.93,-3.29) .. controls (6.95,-1.4) and (3.31,-0.3) .. (0,0) .. controls (3.31,0.3) and (6.95,1.4) .. (10.93,3.29)   ;
\draw    (249.8,170) .. controls (260.07,150.65) and (255.07,175.65) .. (269.9,90.22) ;
\draw    (369.9,91.22) .. controls (386.48,11.62) and (413.27,65.81) .. (414.6,88.96) ;
\draw [shift={(414.65,90.68)}, rotate = 270] [color={rgb, 255:red, 0; green, 0; blue, 0 }  ][line width=0.75]    (10.93,-3.29) .. controls (6.95,-1.4) and (3.31,-0.3) .. (0,0) .. controls (3.31,0.3) and (6.95,1.4) .. (10.93,3.29)   ;
\draw    (349.8,171) .. controls (360.07,151.65) and (355.07,176.65) .. (369.9,91.22) ;

\draw (214.82, 29) node [anchor=center][inner sep=0.75pt]   [align=center] {\ourUB};
\draw (214.82, 46) node [anchor=center][inner sep=0.75pt]   [align=center] {{\small policy-finding}};
\draw (214.82, 59) node [anchor=center][inner sep=0.75pt]   [align=center] {$\horizon = 100$};
\draw (214.82, 100) node [anchor=center][inner sep=0.75pt]   [align=center] {\ourUB};
\draw (214.82, 115) node [anchor=center][inner sep=0.75pt]   [align=center] {{\small hor. reduction}};
\draw (214.82, 130) node [anchor=center][inner sep=0.75pt]   [align=center] {$\horizon = 99$};
\draw (214.82, 170) node [anchor=center][inner sep=0.75pt]   [align=center] {\ourUB};
\draw (214.82, 185) node [anchor=center][inner sep=0.75pt]   [align=center] {{\small heuristic}};
\draw (214.82, 200) node [anchor=center][inner sep=0.75pt]   [align=center] {$\horizon = 5$};
\draw (314.82, 100) node [anchor=center][inner sep=0.75pt]   [align=center] {\ourUB};
\draw (314.82, 115) node [anchor=center][inner sep=0.75pt]   [align=center] {{\small hor. reduction}};
\draw (314.82, 130) node [anchor=center][inner sep=0.75pt]   [align=center] {$\horizon = 94$};
\draw (314.82, 170) node [anchor=center][inner sep=0.75pt]   [align=center] {\ourUB};
\draw (314.82, 185) node [anchor=center][inner sep=0.75pt]   [align=center] {{\small heuristic}};
\draw (314.82, 200) node [anchor=center][inner sep=0.75pt]   [align=center] {$\horizon = 5$};
\draw (415.82, 100) node [anchor=center][inner sep=0.75pt]   [align=center] {\ourUB};
\draw (415.82, 115) node [anchor=center][inner sep=0.75pt]   [align=center] {{\small hor. reduction}};
\draw (415.82, 130) node [anchor=center][inner sep=0.75pt]   [align=center] {$\horizon = 89$};
\draw (415.82, 170) node [anchor=center][inner sep=0.75pt]   [align=center] {\ourUB};
\draw (415.82, 185) node [anchor=center][inner sep=0.75pt]   [align=center] {{\small heuristic}};
\draw (415.82, 200) node [anchor=center][inner sep=0.75pt]   [align=center] {$\horizon = 5$};
\draw (447, 150) node [anchor=center][inner sep=0.75pt]   [align=center] {{\Huge ...}};
\end{tikzpicture}

\caption{Schematic call graph for \ourUB, for $r=5$. A dotted line indicates that a result is provided.}\label{fig:callgraphUB}
\end{figure}
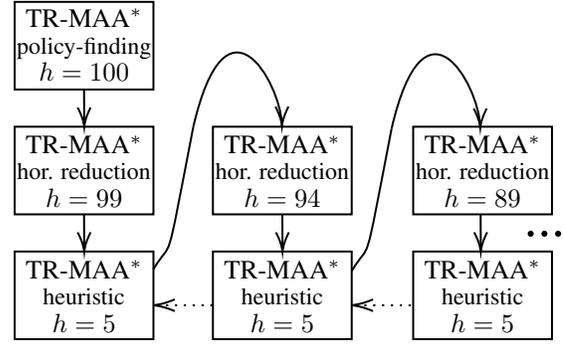

\section{Empirical Evaluation}\label{appE}

In this section we give further details regarding the empirical evaluation. In particular, we discuss the hyperparameter selection and explain how we computed the best known lower/upper bounds provided in the tables. Finally, we give detailed results (including timings) for all benchmarks we tested on and discuss the results further.

\subsubsection{Hyperparameter Selection}

We give additional information on the selection of the hyperparameters. 
For \ourLB, the main hyperparameters are the window size $\windowsize$ for the sliding-window policies (Sec.~\ref{sec:clustering}), 
whether to use the maximum reward heuristic $Q_{\textsf{maxr}, r}$ or the terminal reward MDP heuristic $Q_{\textsf{MDP}, r}$ (Sec.~\ref{sec:loose_heuristics}), the depth $r$ of the heuristic, and the iteration limit $L$ per stage (Sec.~\ref{sec:loose_heuristics} and App.~\ref{appC}). 
Initial experiments indicated that $\windowsize=1$ is too small for most benchmarks, whereas for $\windowsize = 4$, the search space becomes too large, which generally led to worse results than using $\windowsize = 3$.
We also found that $r=1$ clearly guides the search worse than larger $r$, whereas $r \geq 4$ is too expensive (since it needs to be computed for all $\horizon \cdot L$ policies).
This yielded eight different configurations:
Using the heuristics $Q_{\textsf{maxr}, r}$ and $Q_{\textsf{MDP}, r}$ with $\windowsize \in \{2,3\}$ and $r \in \{2,3\}$. 
Based on our tests, we selected  three configurations for the quality runs with a 600 second time limit: ($k=2$, $Q_{\textsf{MDP}, 2}$), ($k=3$, $Q_{\textsf{MDP}, 2}$), and ($k=3$, $Q_{\textsf{maxr}, 3}$). 
For $k=2$ we use $L = 1000$, while for $k=3$ we use $L=10000$. Increasing $L$ yields higher running times, and typically also better results. However, for the smaller search space for $k=2$, using $L = 1000$ is usually sufficient. For the fast runs with a time limit of 60 seconds, we selected the cheapest heuristic $Q_{\textsf{MDP}, 1}$ and chose $L$ small, but still large enough to amply satisfy the constraint $L \geq n|C_{i,t}|$ from Corollary \ref{cor:prune}. In each case, the $Q_3$ heuristic \cite{DBLP:conf/ijcai/Koops0JS23} is used to solve the Dec-POMDPs.

For \ourUB, the main parameter is the depth $r$ and the heuristic $Q_{d'}$ chosen for solving the horizon-$r$ Dec-POMDPs. 
Higher values of $r$ and lower values of $d'$ lead to tighter, more expensive heuristics. For small horizons, it can be better to choose a slightly cheaper heuristic that allows for more exploration in the search tree and hence better bounds for the reward in the initial stages. For larger horizons (e.g.\ 50), it appears in general to be best to choose $r$ as large as possible, while still ensuring that the heuristic for the root of the tree can be computed within the time limit.  We selected $r=5$ with the $Q_3$ heuristic. Like in \rsmaa, $Q_1$ and $Q_2$ are too expensive and therefore require lower values of $r$, while $Q_\infty$ leads to relatively weak bounds on the horizon-$r$ Dec-POMDPs, especially for large $r$.

\subsubsection{Best Known Lower/Upper Bounds}

The best known lower bound (in Table \ref{tab:resultsUB} and \ref{tab:upperall}) was computed as the largest policy value found by FB-HSVI, GA-FSC, or any configuration of \rsmaa or \ourLB used in this paper or in \citet{DBLP:conf/ijcai/Koops0JS23}. The best known upper bound (in Table \ref{tab:resultsLB} and \ref{tab:lowerall}) was computed as the smallest bound computed by any configuration of \rsmaa or \ourUB used in this paper or in \citet{DBLP:conf/ijcai/Koops0JS23}.

We remark that for FB-HSVI, some upper bounds are not valid as we have found policies that exceed these bounds.

\subsubsection{Further Discussion of the Results}

Table \ref{tab:allconfigs} shows the values and timings for each of the configurations for \ourLB.
Table \ref{tab:lowerall} and \ref{tab:upperall} are similar to Table \ref{tab:resultsLB} and \ref{tab:resultsUB} in the main text, but now also show the timings, and list the two baseline configurations for \ourLB using $Q_{\mathsf{MDP}}$ separately in Table \ref{tab:lowerall}. An (f) indicates fast results. We now provide further observations regarding the results.

\paragraph{Providing a solution helps for \ourUB and \rsmaa.} Providing a near-optimal solution can significantly improve the performance of \ourUB and \rsmaa, yielding up to a factor 2 reduction in running time and a factor 4 reduction in (peak) memory usage.

\paragraph{Exact solving.} \ourUB is the first solver to compute the exact value for \grid{} $h = 7$, namely 4.4739533. This paper is also the first to report the exact value for \gridthree{} $h = 7$, namely 2.1923737. However, this result can also be achieved by \rsmaa when providing a near-optimal solution.

\paragraph{$Q_{\mathsf{MDP}, r}$ vs. $Q_{\mathsf{maxr}, r}$.} The terminal reward MDP heuristic $Q_{\mathsf{MDP}, 2}$ works well on all benchmarks except \dectiger for $k=3$. On the other hand, $Q_{\mathsf{maxr}, 3}$ works very well for \dectiger and \boxpush, but poorly on other benchmarks. This is primarily due to the increase to $r=3$: $Q_{\mathsf{maxr}, 2}$ performs poorly on these benchmarks. However, for \dectiger and \boxpush $Q_{\mathsf{maxr}, 3}$ also yields (slightly) better policies than $Q_{\mathsf{MDP}, 3}$, besides being more scalable.

\paragraph{For \grid{}, heuristic reuse is very hard.} For \grid{}, relatively few action-observation histories lead to the same belief. As a result, \rsmaa reaches the memory limit for $h \geq 7$, and also \ourUB reaches the memory limit for the more expensive $r=5$ and higher horizons. 

\paragraph{For \ourUB, $r$ can be benchmark-specific for further improvements.} Usually larger $r$ yield better upper bounds, but some benchmarks have a specific period for stages at which revealing the state does not give too much information. This is for example the case for \mars, where $r=3$ gives better bounds than $r=5$. However, $r=6$ gives better results than $r=3$ (as expected due to the divisibility): upper bounds of 134.38 for $h=50$ and 269.24 for $h = 100$.

\begin{table}[h!]
\resizebox{\columnwidth}{!}{
\addtolength{\tabcolsep}{-0.17253em}
\begin{tabular}{@{}r|rr|rr|rr|rr|rr|rr@{}} \hline
& \multicolumn{2}{|c}{1,$Q_{\mathsf{MDP},1}$,20} & \multicolumn{2}{|c}{2,$Q_{\mathsf{MDP},1}$,100} & \multicolumn{2}{|c}{3,$Q_{\mathsf{MDP},1}$,100} & \multicolumn{2}{|c}{2,$Q_{\mathsf{MDP},2}$,$10^3$} & \multicolumn{2}{|c}{3,$Q_{\mathsf{MDP},2}$,$10^4$} & \multicolumn{2}{|c}{3,$Q_{\mathsf{maxr},3}$,$10^4$} \\ \hline
$h$ & value & time & value & time & value & time & value & time & value & time & value & time \\ \hline
\multicolumn{13}{c}{ \dectiger } \\ \hline
6 &  \textminus 10.68 & \textless 1 & 10.38 & \textless 1 & 10.38 & \textless 1 & 10.38 & 1 & 10.38 & 3 & 10.38 & 1 \\
7 &  \textminus 12.68 & \textless 1 & 8.38 & \textless 1 & 8.37 & \textless 1 & 8.38 & 2 & 7.45 & 4 & 9.99 & 4 \\
8 &  \textminus 14.68 & \textless 1 & 6.38 & \textless 1 & 6.58 & \textless 1 & 7.25 & 2 & 5.05 & 6 & 8.36 & 8 \\
9 &  \textminus 16.68 & \textless 1 & 15.57 & \textless 1 & 10.37 & 1 & 15.57 & 2 & 9.16 & 6 & 15.57 & 6 \\
10 &  \textminus 18.68 & \textless 1 & 13.57 & 1 & 11.66 & 1 & 13.57 & 2 & 8.92 & 7 & 15.18 & 9 \\
12 &  \textminus 22.68 & \textless 1 & 20.76 & 1 & 8.49 & 1 & 20.76 & 2 & 7.58 & 8 & 20.76 & 10 \\
15 &  \textminus 27.55 & \textless 1 & 25.95 & 1 & 7.14 & 1 & 25.95 & 2 & 6.74 & 10 & 25.95 & 12 \\
20 &  \textminus 38.68 & \textless 1 & 27.14 & 1 &  \textminus 0.09 & 1 & 28.01 & 3 & 7.02 & 13 & 29.12 & 17 \\
30 &  \textminus 58.68 & \textless 1 & 51.91 & 1 &  \textminus 18.93 & 1 & 51.91 & 3 & 3.20 & 21 & 51.91 & 17 \\
40 &  \textminus 78.68 & \textless 1 & 65.48 & 1 &  \textminus 31.03 & 1 & 65.48 & 4 & 2.60 & 28 & 67.09 & 22 \\
50 &  \textminus 98.68 & \textless 1 & 79.05 & 1 &  \textminus 49.27 & 2 & 79.92 & 5 & 8.86 & 34 & 81.03 & 29 \\
100 &  \textminus 198.68 & 1 & 169.30 & 1 &  \textminus 131.09 & 2 & 169.30 & 7 & 15.43 & 69 & 170.91 & 47 \\\hline
\multicolumn{13}{c}{ \grid } \\ \hline
4 & 2.176 & \textless 1 & 2.232 & \textless 1 & 2.225 & \textless 1 & 2.241 & 1 & 2.242 & 1 & 2.242 & 1 \\
5 & 2.916 & \textless 1 & 2.915 & 1 & 2.854 & 1 & 2.969 & 2 & 2.970 & 2 & 2.970 & 3 \\
6 & 3.616 & \textless 1 & 3.664 & 1 & 3.623 & 1 & 3.715 & 7 & 3.717 & 7 & 3.717 & 13 \\
7 & 4.319 & \textless 1 & 4.414 & 1 & 4.352 & 1 & 4.471 & 11 & 4.474 & 16 & 4.474 & 34 \\
8 & 5.023 & \textless 1 & 5.168 & 1 & 5.140 & 1 & 5.248 & 18 & 5.229 & 24 & 5.234 & 64 \\
9 & 5.727 & \textless 1 & 5.925 & 1 & 5.898 & 1 & 6.032 & 27 & 6.032 & 41 & 5.999 & 128 \\
10 & 6.431 & \textless 1 & 6.684 & 1 & 6.661 & 1 & 6.817 & 37 & 6.821 & 54 & 6.764 & 164 \\
12 & 7.840 & 1 & 8.203 & 2 & 8.195 & 2 & 8.391 & 57 & 8.414 & 84 & MO & 216 \\
15 & 9.953 & 1 & 10.485 & 2 & 10.507 & 2 & 10.754 & 82 & 10.810 & 154 & MO & 215 \\
20 & 13.475 & 1 & 14.288 & 2 & 14.357 & 2 & 14.693 & 139 & MO & 161 & MO & 215 \\
50 & 34.606 & 2 & 37.105 & 4 & 37.461 & 4 & MO & 155 & MO & 161 & MO & 216 \\
100 & 69.825 & 2 & 75.133 & 6 & 75.966 & 8 & MO & 150 & MO & 159 & MO & 215 \\\hline
\multicolumn{13}{c}{ \boxpush } \\ \hline
4 & 98.17 & \textless 1 & 98.59 & \textless 1 & 98.59 & \textless 1 & 98.59 & 1 & 98.59 & 1 & 98.59 & 1 \\
5 & 107.65 & \textless 1 & 107.73 & \textless 1 & 107.73 & \textless 1 & 107.71 & 1 & 107.73 & 1 & 107.69 & 3 \\
6 & 114.40 & \textless 1 & 115.90 & \textless 1 & 115.63 & \textless 1 & 119.74 & 1 & 120.00 & 4 & 120.68 & 8 \\
7 & 150.09 & \textless 1 & 155.25 & 1 & 155.99 & 1 & 155.35 & 2 & 155.99 & 3 & 155.91 & 9 \\
8 & 186.03 & \textless 1 & 187.09 & 1 & 191.20 & 1 & 187.09 & 2 & 191.20 & 5 & 191.23 & 10 \\
9 & 203.06 & \textless 1 & 203.87 & 1 & 210.26 & 1 & 203.87 & 2 & 210.26 & 5 & 210.26 & 11 \\
10 & 218.32 & \textless 1 & 219.12 & 1 & 221.12 & 1 & 219.12 & 2 & 223.70 & 5 & 224.28 & 11 \\
12 & 265.11 & \textless 1 & 266.77 & 1 & 282.70 & 1 & 270.59 & 3 & 283.79 & 8 & 284.15 & 13 \\
15 & 326.87 & \textless 1 & 325.52 & 1 & 342.99 & 1 & 324.88 & 2 & 346.95 & 10 & 349.18 & 14 \\
20 & 421.35 & 1 & 423.19 & 1 & 466.34 & 1 & 420.94 & 3 & 471.64 & 14 & 474.97 & 17 \\
30 & 625.60 & 1 & 612.03 & 1 & 718.76 & 2 & 619.01 & 4 & 720.68 & 59 & 721.28 & 22 \\
40 & 844.82 & 1 & 811.27 & 1 & 963.02 & 2 & 815.93 & 5 & 964.54 & 121 & 965.09 & 27 \\
50 & 991.02 & 1 & 977.79 & 2 & 1207.67 & 3 & 1012.03 & 5 & 1209.26 & 105 & 1209.80 & 33 \\
100 & 1905.81 & 2 & 1892.52 & 2 & 2431.40 & 5 & 1993.41 & 9 & 2432.96 & 331 & 2433.51 & 61 \\\hline
\multicolumn{13}{c}{ \ff } \\ \hline
4 &  \textminus 6.626 & \textless 1 &  \textminus 6.672 & \textless 1 &  \textminus 6.586 & \textless 1 &  \textminus 6.626 & 1 &  \textminus 6.579 & 1 &  \textminus 6.579 & 1 \\
5 &  \textminus 7.096 & \textless 1 &  \textminus 7.294 & 1 &  \textminus 7.148 & 1 &  \textminus 7.095 & 3 &  \textminus 7.073 & 3 &  \textminus 7.073 & 4 \\
6 &  \textminus 7.379 & \textless 1 &  \textminus 7.661 & 1 &  \textminus 7.673 & 1 &  \textminus 7.196 & 1 &  \textminus 7.176 & 1 &  \textminus 7.176 & 14 \\
10 &  \textminus 7.196 & \textless 1 &  \textminus 7.196 & \textless 1 &  \textminus 7.255 & 1 &  \textminus 7.196 & 1 &  \textminus 7.176 & 1 &  \textminus 7.176 & 14 \\
100 &  \textminus 7.196 & \textless 1 &  \textminus 7.196 & 1 &  \textminus 7.256 & 1 &  \textminus 7.196 & 1 &  \textminus 7.176 & 1 &  \textminus 7.176 & 14 \\
2000 &  \textminus 7.196 & 2 &  \textminus 7.196 & 2 &  \textminus 7.256 & 3 &  \textminus 7.196 & 3 &  \textminus 7.176 & 3 &  \textminus 7.176 & 15 \\\hline
\multicolumn{13}{c}{ \gridthree } \\ \hline
5 & 0.90 & \textless 1 & 0.90 & 1 & 0.90 & 1 & 0.90 & 1 & 0.90 & 1 & 0.90 & 7 \\
6 & 1.49 & \textless 1 & 1.49 & 1 & 1.49 & 1 & 1.49 & 2 & 1.49 & 3 & 1.49 & 16 \\
7 & 2.17 & 1 & 2.17 & 1 & 2.17 & 1 & 2.19 & 4 & 2.19 & 7 & 2.18 & 19 \\
8 & 2.93 & 1 & 2.96 & 1 & 2.96 & 1 & 2.97 & 6 & 2.97 & 11 & 2.92 & 29 \\
9 & 3.75 & 1 & 3.80 & 1 & 3.80 & 1 & 3.81 & 9 & 3.81 & 21 & 3.70 & 34 \\
10 & 4.60 & 1 & 4.67 & 1 & 4.67 & 1 & 4.68 & 12 & 4.68 & 23 & 4.50 & 46 \\
12 & 6.50 & 1 & 6.50 & 1 & 6.50 & 1 & 6.53 & 18 & 6.53 & 45 & 6.13 & 56 \\
15 & 9.40 & 1 & 9.39 & 2 & 9.39 & 2 & 9.42 & 26 & 9.42 & 70 & 8.61 & 83 \\
20 & 14.34 & 1 & 14.33 & 2 & 14.33 & 3 & 14.37 & 46 & 14.37 & 122 & 12.76 & 123 \\
30 & 24.32 & 1 & 24.32 & 3 & 24.32 & 4 & 24.35 & 69 & 24.35 & 246 & 21.07 & 201 \\
40 & 34.32 & 1 & 34.32 & 4 & 34.32 & 6 & 34.35 & 102 & 34.35 & 416 & 29.39 & 273 \\
50 & 44.32 & 1 & 44.32 & 6 & 44.32 & 8 & 44.35 & 130 & 44.35 & 556 & 37.70 & 366 \\
100 & 94.32 & 2 & 94.32 & 9 & 94.32 & 14 & 94.35 & 202 & TO & 600 & 79.26 & 554 \\\hline
\multicolumn{13}{c}{ \mars } \\ \hline
6 & 16.98 & \textless 1 & 16.99 & 1 & 16.99 & 1 & 18.62 & 2 & 18.62 & 3 & 18.62 & 3 \\
7 & 20.60 & \textless 1 & 20.60 & 1 & 20.60 & 1 & 20.90 & 2 & 20.90 & 3 & 20.90 & 10 \\
8 & 20.57 & \textless 1 & 20.65 & 1 & 20.41 & 1 & 22.48 & 2 & 22.48 & 4 & 21.13 & 12 \\
9 & 23.92 & \textless 1 & 24.02 & 1 & 24.02 & 1 & 24.31 & 3 & 24.32 & 5 & 20.99 & 12 \\
10 & 23.58 & 1 & 24.44 & 1 & 24.44 & 1 & 26.31 & 4 & 26.31 & 14 & 20.80 & 12 \\
12 & 30.94 & 1 & 31.01 & 1 & 31.00 & 1 & 32.09 & 4 & 32.05 & 25 & 20.40 & 13 \\
15 & 33.66 & 1 & 34.23 & 1 & - & 2 & 37.72 & 6 & 37.47 & 53 & 19.80 & 15 \\
20 & 49.08 & 1 & 44.87 & 2 & - & 1 & 49.37 & 13 & 49.15 & 81 & 18.80 & 20 \\
30 & 74.54 & 1 & 75.15 & 3 & - & 2 & 77.47 & 21 & 77.21 & 199 & 16.80 & 36 \\
40 & 95.39 & 1 & 96.06 & 3 & - & 1 & 100.02 & 25 & 94.12 & 269 & 14.80 & 91 \\
50 & 116.35 & 1 & 117.03 & 4 & - & 2 & 122.56 & 37 & 117.33 & 266 & 29.40 & 376 \\
100 & 221.14 & 2 & 221.97 & 9 & - & 2 & 234.06 & 89 & TO & 600 & TO & 600 \\\hline
\multicolumn{13}{c}{ \hotel } \\ \hline
100 & 502.2 & \textless 1 & 502.2 & 1 & 502.2 & 1 & 502.2 & 3 & 502.2 & 19 & 502.2 & 18 \\
500 & 2502.2 & 1 & 2502.2 & 2 & 2502.2 & 2 & 2502.2 & 21 & MO & 22 & 2502.2 & 138 \\
1000 & 5002.2 & 2 & 5002.2 & 5 & 5002.2 & 5 & MO & 22 & MO & 21 & 5002.2 & 388 \\
2000 & 10002.2 & 4 & 10002.2 & 16 & 10002.2 & 17 & MO & 20 & MO & 21 & TO & 600 \\\hline
\multicolumn{13}{c}{ \recycling } \\ \hline
100 & 308.79 & 1 & 308.79 & 1 & 308.79 & 1 & 308.79 & 5 & 308.79 & 34 & 308.79 & 29 \\
500 & 1539.56 & 1 & 1539.56 & 2 & 1539.56 & 2 & 1539.56 & 19 & 1539.56 & 183 & 1539.56 & 158 \\
1000 & 3078.02 & 1 & 3078.02 & 4 & 3078.02 & 4 & 3078.02 & 46 & 3078.02 & 511 & 3078.02 & 470 \\
2000 & 6154.94 & 3 & 6154.94 & 10 & 6154.94 & 9 & 6154.94 & 135 & TO & 600 & TO & 600 \\\hline
\multicolumn{13}{c}{ \broadcast } \\ \hline
100 & 90.74 & \textless 1 & 90.56 & 1 & 90.56 & 1 & 90.76 & 2 & 90.75 & 17 & 90.76 & 14 \\
500 & 452.69 & 1 & 450.56 & 2 & 450.56 & 2 & 452.72 & 11 & 452.73 & 117 & 452.05 & 100 \\
1000 & 905.09 & 1 & 900.56 & 3 & 900.56 & 3 & 905.18 & 28 & MO & 222 & 903.17 & 287 \\
2000 & 1809.87 & 2 & 1800.56 & 10 & 1800.56 & 9 & 1810.07 & 89 & MO & 200 & TO & 600 \\\hline
\end{tabular}
}
\caption{Results for each configuration of \ourLB listed separately. MO and TO denote memout (\textgreater 16GB) and timeout (\textgreater 600s). For \mars, with $k=3$ and $L=100$ there are more clusters than iterations for $h \geq 15$, and hence the solver does not provide a result. } \label{tab:allconfigs}
\end{table}

\newpage

\begin{table}
\resizebox{\columnwidth}{!}{
\addtolength{\tabcolsep}{-0.3747em}
\begin{tabular}{@{}r|rr|rr|rr|rr|rr|rr@{}} \hline
 & FB-HSVI & GA-FSC & \multicolumn{2}{|c}{2,$Q_{\mathsf{MDP}}$,$10^3$} & \multicolumn{2}{|c}{3,$Q_{\mathsf{MDP}}$,$10^4$} & \multicolumn{2}{|c}{\ourLB(f)} & \multicolumn{2}{|c|}{\ourLB} & random  & upper \\ \hline
$h$ & value & value & value & time & value & time & value & time & value & time & value & value \\ \hline
\multicolumn{13}{c}{ \dectiger } \\ \hline
6 & 10.38 & 10.38 & 10.38 & 1 & 10.38 & 2 & 10.38 & 1 & 10.38 & 5 &  \textminus 277.3 & 10.38 \\ 
7 & 9.99 &  & 8.38 & 1 & 8.38 & 2 & 8.38 & 1 & 9.99 & 10 &  \textminus 323.6 & 9.99 \\ 
8 & 12.22 & 8.00 & 7.25 & 1 & 6.38 & 3 & 6.58 & 1 & 8.36 & 16 &  \textminus 369.8 & 12.22 \\ 
9 & 15.57 &  & 15.57 & 1 & 15.57 & 3 & 15.57 & 1 & 15.57 & 14 &  \textminus 416.0 & 15.57 \\ 
10 & 15.18 & 13.57 & 13.57 & 1 & 13.57 & 4 & 13.57 & 1 & 15.18 & 18 &  \textminus 462.2 & 15.18 \\ 
12 &  &  & 20.76 & 1 & 20.76 & 4 & 20.76 & 1 & 20.76 & 20 &  \textminus 554.7 & 20.76 \\ 
15 &  &  & 25.95 & 1 & 25.95 & 5 & 25.95 & 2 & 25.95 & 25 &  \textminus 693.3 & 25.95 \\ 
20 & 28.75 &  & 28.01 & 1 & 27.14 & 5 & 27.14 & 2 & 29.12 & 33 &  \textminus 924.4 & 38.28 \\ 
30 & 51.90 &  & 51.91 & 2 & 51.91 & 7 & 51.91 & 2 & 51.91 & 42 &  \textminus 1386.7 & 59.47 \\ 
40 & 67.09 &  & 65.48 & 2 & 65.48 & 8 & 65.48 & 3 & 67.09 & 54 &  \textminus 1848.9 & 80.35 \\ 
50 & 80.66 &  & 79.92 & 2 & 79.05 & 10 & 79.05 & 3 & 81.03 & 67 &  \textminus 2311.1 & 101.32 \\ 
100 & 170.90 & 169.30 & 169.30 & 2 & 169.30 & 18 & 169.30 & 4 & 170.91 & 123 &  \textminus 4622.2 & 206.44 \\ \hline
\multicolumn{13}{c}{ \grid } \\ \hline
4 & 2.24 & 2.24 & 2.240 & \textless 1 & 2.242 & 1 & 2.232 & 1 & 2.242 & 3 & 0.684 & 2.242 \\ 
5 & 2.97 & 2.97 & 2.960 & 1 & 2.957 & 2 & 2.916 & 2 & 2.970 & 7 & 0.928 & 2.970 \\ 
6 & 3.71 & 3.72 & 3.696 & 1 & 3.689 & 2 & 3.664 & 2 & 3.717 & 27 & 1.176 & 3.717 \\ 
7 & 4.47 &  & 4.441 & 1 & 4.429 & 3 & 4.414 & 2 & 4.474 & 61 & 1.425 & 4.474 \\ 
8 & 5.23 &  & 5.190 & 1 & 5.170 & 3 & 5.168 & 3 & 5.248 & 106 & 1.675 & 5.351 \\ 
9 & 6.03 &  & 5.942 & 1 & 5.913 & 3 & 5.925 & 3 & 6.032 & 195 & 1.924 & 6.277 \\ 
10 &  &  & 6.695 & 1 & 6.657 & 3 & 6.684 & 3 & 6.821 & 256 & 2.174 & 7.222 \\ 
12 &  &  & 8.202 & 1 & 8.146 & 4 & 8.203 & 4 & 8.414 & 356 & 2.674 & 9.172 \\ 
15 &  &  & 10.464 & 2 & 10.382 & 4 & 10.507 & 4 & 10.810 & 450 & 3.424 & 12.143 \\ 
20 &  &  & 14.234 & 2 & 14.108 & 5 & 14.357 & 5 & 14.693 & 515 & 4.674 & 17.132 \\ 
50 &  & 40.49 & 36.858 & 3 & 36.463 & 11 & 37.461 & 9 & MO & 532 & 12.174 & 47.210 \\ 
100 &  &  & 74.564 & 4 & 73.721 & 23 & 75.966 & 16 & MO & 524 & 24.674 & 97.564 \\ \hline
\multicolumn{13}{c}{ \boxpush } \\ \hline
4 & 98.59 & 98.17 & 98.59 & \textless 1 & 98.59 & 1 & 98.59 & 1 & 98.59 & 2 &  \textminus 1.69 & 98.59 \\ 
5 & 107.72 &  & 107.73 & \textless 1 & 107.73 & 1 & 107.73 & 1 & 107.73 & 5 &  \textminus 2.67 & 107.73 \\ 
6 & 120.67 &  & 119.11 & 1 & 119.78 & 2 & 115.90 & 1 & 120.68 & 13 &  \textminus 3.73 & 120.69 \\ 
7 & 156.42 &  & 155.74 & 1 & 155.99 & 2 & 155.99 & 1 & 155.99 & 14 &  \textminus 4.84 & 156.71 \\ 
8 & 191.22 & 189.05 & 186.86 & 1 & 191.07 & 2 & 191.20 & 1 & 191.23 & 18 &  \textminus 5.97 & 191.29 \\ 
9 & 210.27 &  & 203.87 & 1 & 208.48 & 3 & 210.26 & 1 & 210.26 & 18 &  \textminus 7.12 & 210.35 \\ 
10 & 223.74 & 216.98 & 218.33 & 1 & 220.35 & 3 & 221.12 & 2 & 224.28 & 19 &  \textminus 8.30 & 224.51 \\ 
12 &  &  & 260.64 & 1 & 270.01 & 4 & 282.70 & 2 & 284.15 & 24 &  \textminus 10.67 & 285.08 \\ 
15 &  &  & 310.47 & 1 & 316.76 & 7 & 342.99 & 2 & 349.18 & 26 &  \textminus 14.31 & 350.37 \\ 
20 & 458.10 & 468.15 & 402.46 & 1 & 393.19 & 11 & 466.34 & 3 & 474.97 & 34 &  \textminus 20.46 & 476.43 \\ 
30 & 636.28 &  & 579.44 & 1 & 497.20 & 28 & 718.76 & 4 & 721.28 & 85 &  \textminus 32.91 & 724.78 \\ 
40 & 774.14 &  & 766.74 & 2 & 560.54 & 41 & 963.02 & 5 & 965.09 & 153 &  \textminus 45.42 & 971.40 \\ 
50 & 1134.70 & 1201.05 & 949.75 & 2 & 599.58 & 53 & 1207.67 & 6 & 1209.80 & 143 &  \textminus 57.94 & 1218.39 \\ 
100 &  & 2420.26 & 1864.49 & 3 & 651.03 & 93 & 2431.40 & 9 & 2433.51 & 401 &  \textminus 120.55 & 2453.43 \\ \hline
\multicolumn{13}{c}{ \ff } \\ \hline
4 &  & \textminus6.58 &  \textminus 6.702 & 1 &  \textminus 6.579 & 1 &  \textminus 6.586 & 1 &  \textminus 6.579 & 3 &  \textminus 9.026 & \textminus6.579 \\ 
5 &  & \textminus7.07 &  \textminus 7.095 & 1 &  \textminus 7.073 & 2 &  \textminus 7.096 & 2 &  \textminus 7.073 & 9 &  \textminus 10.849 & \textminus7.070 \\ 
6 &  & \textminus7.18 &  \textminus 7.458 & 1 &  \textminus 7.176 & 2 &  \textminus 7.379 & 2 &  \textminus 7.176 & 17 &  \textminus 12.556 & \textminus7.176 \\ 
10 &  &  &  \textminus 7.612 & 1 &  \textminus 7.176 & 1 &  \textminus 7.196 & 1 &  \textminus 7.176 & 15 &  \textminus 18.413 & \textminus7.176 \\ 
100 &  &  &  \textminus 7.612 & 1 &  \textminus 7.176 & 2 &  \textminus 7.196 & 2 &  \textminus 7.176 & 16 &  \textminus 40.028 & \textminus7.176 \\ 
2000 &  &  &  \textminus 7.612 & 2 &  \textminus 7.176 & 3 &  \textminus 7.196 & 7 &  \textminus 7.176 & 20 &  \textminus 40.130 & \textminus 7.176 \\ \hline
\multicolumn{13}{c}{ \gridthree } \\ \hline
5 & 0.89 &  & 0.90 & \textless 1 & 0.90 & \textless 1 & 0.90 & 1 & 0.90 & 9 & 0.02 & 0.90 \\ 
6 & 1.49 &  & 1.49 & 1 & 1.49 & 3 & 1.49 & 2 & 1.49 & 20 & 0.03 & 1.49 \\ 
7 & 2.19 &  & 2.19 & 1 & 2.19 & 12 & 2.17 & 2 & 2.19 & 30 & 0.05 & 2.19 \\ 
8 & 2.96 &  & 2.97 & 2 & 2.97 & 19 & 2.96 & 3 & 2.97 & 45 & 0.07 & 2.97 \\ 
9 & 3.80 &  & 3.80 & 3 & 3.80 & 24 & 3.80 & 2 & 3.81 & 64 & 0.09 & 3.81 \\ 
10 & 4.68 &  & 4.67 & 4 & 4.68 & 31 & 4.67 & 3 & 4.68 & 82 & 0.12 & 4.70 \\ 
12 &  &  & 6.51 & 4 & 6.51 & 40 & 6.50 & 3 & 6.53 & 119 & 0.16 & 6.55 \\ 
15 &  &  & 9.40 & 6 & 9.41 & 68 & 9.40 & 4 & 9.42 & 180 & 0.24 & 9.45 \\ 
20 & 14.35 &  & 14.35 & 8 & 14.35 & 103 & 14.34 & 6 & 14.37 & 292 & 0.36 & 14.39 \\ 
30 & 24.33 &  & 24.33 & 13 & 24.34 & 226 & 24.32 & 8 & 24.35 & 516 & 0.61 & 24.38 \\ 
40 & 34.33 &  & 34.33 & 19 & 34.33 & 355 & 34.32 & 12 & 34.35 & 791 & 0.85 & 34.38 \\ 
50 & 44.32 &  & 44.33 & 25 & 44.33 & 544 & 44.32 & 15 & 44.35 & 1052 & 1.10 & 44.38 \\ 
100 & 94.24 &  & 94.33 & 59 & TO & 600 & 94.32 & 25 & 94.35 & 1356 & 2.34 & 94.38 \\ \hline
\multicolumn{13}{c}{ \mars } \\ \hline
6 & 18.62 &  & 16.99 & 2 & 17.84 & 7 & 16.99 & 2 & 18.62 & 8 &  \textminus 8.58 & 18.62 \\ 
7 & 20.90 &  & 20.71 & 1 & 20.71 & 6 & 20.60 & 2 & 20.90 & 14 &  \textminus 9.86 & 20.90 \\ 
8 & 22.47 &  & 21.59 & 2 & 21.59 & 11 & 20.65 & 2 & 22.48 & 18 &  \textminus 11.11 & 22.48 \\ 
9 & 24.31 &  & 24.11 & 2 & 24.13 & 12 & 24.02 & 2 & 24.32 & 21 &  \textminus 12.34 & 24.32 \\ 
10 & 26.31 &  & 24.63 & 2 & 24.63 & 15 & 24.44 & 2 & 26.31 & 31 &  \textminus 13.56 & 26.32 \\ 
12 &  &  & 31.12 & 2 & 31.12 & 17 & 31.01 & 3 & 32.09 & 42 &  \textminus 15.98 & 33.64 \\ 
15 &  &  & 34.33 & 3 & 35.25 & 19 & 34.23 & 4 & 37.72 & 74 &  \textminus 19.59 & 39.48 \\ 
20 & 52.13 &  & 45.30 & 5 & 46.09 & 40 & 49.08 & 4 & 49.37 & 113 &  \textminus 25.63 & 52.71 \\ 
30 & 78.09 &  & 74.95 & 6 & 74.62 & 107 & 75.15 & 5 & 77.47 & 256 &  \textminus 37.76 & 79.50 \\ 
40 & 103.52 &  & 95.76 & 8 & 93.72 & 309 & 96.06 & 5 & 100.02 & 385 &  \textminus 49.89 & 106.16 \\ 
50 & 128.95 &  & 116.73 & 17 & 111.22 & 511 & 117.03 & 8 & 122.56 & 678 &  \textminus 62.02 & 132.82 \\ 
100 & 249.92 &  & 221.70 & 31 & TO & 600 & 221.97 & 13 & 234.06 & 1289 &  \textminus 122.67 & 265.73 \\ \hline
\multicolumn{13}{c}{ \hotel } \\ \hline
100 &  &  & 502.2 & 2 & 502.2 & 15 & 502.2 & 2 & 502.2 & 39 &  \textminus 184.6 & 502.2 \\ 
500 &  &  & 2502.2 & 13 & MO & 21 & 2502.2 & 5 & 2502.2 & 181 &  \textminus 1095.7 & 2502.2 \\ 
1000 &  &  & MO & 19 & MO & 18 & 5002.2 & 12 & 5002.2 & 431 &  \textminus 2234.6 & 5002.2 \\ 
2000 &  &  & MO & 18 & MO & 17 & 10002.2 & 37 & MO & 641 &  \textminus 4512.4 & 10002.2 \\ \hline
\multicolumn{13}{c}{ \recycling } \\ \hline
100 & 308.78 & 308.79 & 308.79 & 3 & 308.79 & 24 & 308.79 & 3 & 308.79 & 67 & 47.36 & 308.79 \\ 
500 &  &  & 1539.56 & 11 & 1539.56 & 145 & 1539.56 & 5 & 1539.56 & 361 & 229.32 & 1539.56 \\ 
1000 &  &  & 3078.02 & 31 & 3078.02 & 372 & 3078.02 & 9 & 3078.02 & 1026 & 456.77 & 3078.02 \\ 
2000 &  &  & 6154.94 & 108 & TO & 600 & 6154.94 & 22 & 6154.94 & 1335 & 911.68 & 6154.94 \\ \hline
\multicolumn{13}{c}{ \broadcast } \\ \hline
100 & 90.76 & 90.29 & 90.76 & 1 & 90.75 & 15 & 90.74 & 2 & 90.76 & 32 & 28.62 & 90.76 \\ 
500 &  & 450.29 & 452.72 & 9 & 452.73 & 109 & 452.69 & 4 & 452.73 & 228 & 141.53 & 452.74 \\ 
1000 & 903.28 & 900.29 & 905.16 & 25 & MO & 189 & 905.09 & 8 & 905.18 & 537 & 282.68 & 905.20 \\ 
2000 &  &  & 1810.02 & 80 & MO & 199 & 1809.87 & 21 & 1810.07 & 889 & 564.98 & 1810.13 \\ \hline
\end{tabular}
}
\caption{Lower bounds, i.e.\ the value of the policies obtained with different policy-finding algorithms. MO denotes memout (\textgreater 16GB), TO denotes timeout (\textgreater 600s).} \label{tab:lowerall}
\end{table}

\begin{table}
\resizebox{\columnwidth}{!}{
\addtolength{\tabcolsep}{-0.19977em}
\begin{tabular}{@{}r|r|rr|rr|rr|rr|rr@{}} \hline
& FB-HSVI & \multicolumn{2}{|c}{\rsmaa(f)} & \multicolumn{2}{|c}{\rsmaa} & \multicolumn{2}{|c}{\ourUB(f)} & \multicolumn{2}{|c|}{\ourUB} & lower & MDP \\ \hline
$h$ & value & value & time & value & time & value & time & value & time & value & value \\ \hline
\multicolumn{12}{c}{ \dectiger } \\ \hline
6 & 10.39 & 10.38 & \textless 1 & 10.38 & \textless 1 & 10.38 & \textless 1 & 10.38 & \textless 1 & 10.38 & 120.00 \\ 
7 & 10.00 & 9.99 & 1 & 9.99 & 1 & 11.27 & 120 & 9.99 & 1 & 9.99 & 140.00 \\ 
8 & 12.23 & 12.22 & 1 & 12.22 & 1 & 16.90 & 120 & 12.22 & 2 & 12.22 & 160.00 \\ 
9 & 15.58 & 15.57 & 5 & 15.57 & 5 & 22.79 & 120 & 15.57 & 4 & 15.57 & 180.00 \\ 
10 & 15.19 & 15.27 & 120 & 15.18 & 349 & 27.28 & 120 & 19.59 & 243 & 15.18 & 200.00 \\ 
12 &  & 20.76 & 120 & 20.76 & 453 & 38.64 & 120 & 27.99 & 323 & 20.76 & 240.00 \\ 
15 &  & 27.43 & 120 & 25.95 & 1800 & 55.39 & 120 & 42.52 & 205 & 25.95 & 300.00 \\ 
20 &  & 38.28 & 120 & 38.28 & 141 & 81.07 & 120 & 65.71 & 194 & 29.12 & 400.00 \\ 
30 &  & 59.47 & 120 & 59.47 & 131 & 135.45 & 120 & 112.05 & 204 & 51.91 & 600.00 \\ 
40 &  & 80.35 & 120 & 80.35 & 131 & 188.13 & 120 & 158.46 & 186 & 67.09 & 800.00 \\ 
50 &  & 101.32 & 120 & 101.32 & 136 & 241.07 & 120 & 204.84 & 184 & 81.03 & 1000.00 \\ 
100 &  & 206.44 & 120 & 206.44 & 145 & 508.13 & 120 & 436.73 & 197 & 170.91 & 2000.00 \\ \hline
\multicolumn{12}{c}{ \grid } \\ \hline
4 & 2.25 & 2.242 & 1 & 2.242 & 1 & 2.242 & 2 & 2.242 & 1 & 2.242 & 2.865 \\ 
5 & 2.98 & 2.970 & 3 & 2.970 & 3 & 2.970 & 4 & 2.970 & 7 & 2.970 & 3.834 \\ 
6 & 3.72 & 3.717 & 43 & 3.717 & 42 & 3.717 & 10 & 3.717 & 71 & 3.720 & 4.820 \\ 
7 & 4.48 & TO & 120 & MO & 186 & 4.474 & 56 & 4.491 & 274 & 4.474 & 5.814 \\ 
8 & 5.24 & TO & 120 & MO & 190 & 5.351 & 120 & 5.404 & 268 & 5.248 & 6.811 \\ 
9 & 6.04 & TO & 120 & MO & 198 & 6.277 & 120 & 6.448 & 240 & 6.032 & 7.809 \\ 
10 &  & TO & 120 & MO & 201 & 7.222 & 120 & 7.387 & 262 & 6.821 & 8.809 \\ 
12 &  & TO & 120 & MO & 180 & 9.172 & 120 & MO & 261 & 8.414 & 10.808 \\ 
15 &  & TO & 120 & MO & 198 & 12.143 & 120 & MO & 284 & 10.810 & 13.808 \\ 
20 &  & TO & 120 & MO & 203 & 17.132 & 120 & MO & 287 & 14.693 & 18.808 \\ 
50 &  & TO & 120 & MO & 185 & 47.210 & 120 & MO & 285 & 40.490 & 48.808 \\ 
100 &  & TO & 120 & MO & 187 & 97.564 & 120 & MO & 290 & 75.966 & 98.808 \\ \hline
\multicolumn{12}{c}{ \boxpush } \\ \hline
4 & 98.60 & 98.59 & 5 & 98.59 & 5 & 98.59 & 6 & 98.59 & 6 & 98.59 & 106.43 \\ 
5 & 107.73 & 107.74 & 120 & 107.74 & 1800 & 107.74 & 120 & 107.74 & 1800 & 107.73 & 118.77 \\ 
6 & 120.68 & 120.72 & 93 & 120.69 & 591 & 120.72 & 95 & 120.69 & 588 & 120.68 & 131.19 \\ 
7 & 156.43 & 156.72 & 120 & 156.71 & 271 & 156.89 & 120 & 156.71 & 272 & 156.42 & 163.94 \\ 
8 & 191.23 & 191.29 & 120 & 191.29 & 240 & 191.40 & 120 & 191.29 & 235 & 191.23 & 203.42 \\ 
9 & 210.28 & 210.35 & 120 & 210.35 & 306 & 210.93 & 114 & 210.37 & 210 & 210.27 & 228.75 \\ 
10 & 223.75 & 227.31 & 120 & 224.51 & 455 & 225.15 & 120 & 224.66 & 239 & 224.28 & 244.85 \\ 
12 &  & TO & 120 & MO & 248 & 285.77 & 120 & 285.08 & 249 & 284.15 & 302.73 \\ 
15 &  & TO & 120 & MO & 247 & 354.48 & 120 & 350.37 & 251 & 349.18 & 377.12 \\ 
20 &  & TO & 120 & MO & 251 & 481.16 & 120 & 476.43 & 260 & 474.97 & 511.13 \\ 
30 &  & TO & 120 & MO & 251 & 730.55 & 120 & 724.78 & 327 & 721.28 & 777.61 \\ 
40 &  & TO & 120 & MO & 248 & 978.11 & 120 & 971.40 & 396 & 965.09 & 1042.03 \\ 
50 &  & TO & 120 & MO & 245 & 1227.43 & 120 & 1218.39 & 493 & 1209.80 & 1306.24 \\ 
100 &  & TO & 120 & MO & 248 & 2469.74 & 120 & 2453.43 & 963 & 2433.51 & 2628.14 \\ \hline
\multicolumn{12}{c}{ \ff } \\ \hline
4 &  &  \textminus 6.579 & 1 &  \textminus 6.579 & 1 &  \textminus 6.579 & 3 &  \textminus 6.579 & 1 & \textminus6.579 & \textminus4.282 \\ 
5 &  &  \textminus 7.070 & 6 &  \textminus 7.070 & 6 &  \textminus 7.070 & 8 &  \textminus 7.070 & 6 & \textminus7.070 & \textminus4.342 \\ 
6 &  &  \textminus 7.176 & 40 &  \textminus 7.176 & 40 &  \textminus 7.176 & 6 &  \textminus 6.390 & 313 & \textminus7.176 & \textminus4.355 \\ 
10 &  & TO & 120 & MO & 245 &  \textminus 7.176 & 10 &  \textminus 6.432 & 334 & \textminus7.176 & \textminus4.363 \\ 
100 &  & TO & 120 & MO & 249 &  \textminus 7.176 & 105 & MO & 365 & \textminus7.176 & \textminus4.363 \\ 
2000 &  & TO & 120 & MO & 250 & TO & 120 & MO & 344 & \textminus7.176 & \textminus4.363 \\ \hline
\multicolumn{12}{c}{ \gridthree } \\ \hline
5 & 0.90 & 0.90 & 1 & 0.90 & 1 & 0.90 & 1 & 0.90 & 1 & 0.90 & 0.94 \\ 
6 & 1.50 & 1.49 & 6 & 1.49 & 5 & 1.49 & 7 & 1.49 & 6 & 1.49 & 1.58 \\ 
7 & 2.20 & 2.19 & 120 & 2.19 & 130 & 2.19 & 120 & 2.19 & 144 & 2.19 & 2.31 \\ 
8 & 2.97 & 2.97 & 120 & 2.97 & 225 & 2.97 & 120 & 2.97 & 303 & 2.97 & 3.12 \\ 
9 & 3.81 & 3.81 & 120 & 3.81 & 222 & 3.81 & 120 & 3.81 & 287 & 3.81 & 3.98 \\ 
10 & 4.69 & 4.70 & 120 & 4.70 & 241 & 4.70 & 120 & 4.70 & 289 & 4.68 & 4.88 \\ 
12 &  & 6.55 & 120 & 6.55 & 264 & 6.55 & 120 & 6.55 & 317 & 6.53 & 6.76 \\ 
15 &  & 9.45 & 120 & 9.45 & 261 & 9.45 & 120 & 9.45 & 373 & 9.42 & 9.67 \\ 
20 & 14.36 & 14.40 & 120 & 14.39 & 266 & 14.40 & 120 & 14.39 & 441 & 14.37 & 14.63 \\ 
30 & 24.34 & 24.39 & 120 & 24.38 & 336 & 24.39 & 120 & 24.38 & 663 & 24.35 & 24.62 \\ 
40 &  & 34.39 & 120 & 34.38 & 679 & 34.39 & 120 & 34.38 & 764 & 34.35 & 34.62 \\ 
50 &  & 44.39 & 120 & 44.38 & 438 & 44.44 & 120 & 44.38 & 936 & 44.35 & 44.62 \\ 
100 &  & 94.39 & 120 & 94.38 & 702 & 94.45 & 120 & 94.38 & 1753 & 94.35 & 94.62 \\ \hline
\multicolumn{12}{c}{ \mars } \\ \hline
6 & 18.63 & 18.62 & 2 & 18.62 & 2 & 18.62 & 2 & 18.62 & 3 & 18.62 & 20.07 \\ 
7 & 20.91 & 20.90 & 3 & 20.90 & 3 & 20.90 & 4 & 20.90 & 14 & 20.90 & 21.17 \\ 
8 & 22.48 & 22.48 & 25 & 22.48 & 25 & 22.48 & 16 & 22.48 & 18 & 22.48 & 23.65 \\ 
9 & 24.32 & 24.32 & 75 & 24.32 & 43 & 24.32 & 6 & 24.32 & 56 & 24.32 & 25.70 \\ 
10 & 26.32 & 26.32 & 85 & 26.32 & 1800 & 26.32 & 113 & 26.32 & 373 & 26.31 & 28.61 \\ 
12 &  & 33.85 & 120 & 33.64 & 1800 & 33.96 & 112 & 33.92 & 1800 & 32.09 & 36.11 \\ 
15 &  & 39.56 & 69 & 39.48 & 252 & 39.71 & 120 & 39.53 & 329 & 37.72 & 43.02 \\ 
20 &  & 52.86 & 81 & 52.71 & 559 & 53.78 & 120 & 53.11 & 575 & 52.13 & 57.52 \\ 
30 &  & 79.50 & 98 & 79.50 & 98 & 81.08 & 120 & 80.62 & 803 & 78.09 & 86.65 \\ 
40 &  & 106.16 & 116 & 106.16 & 116 & 108.89 & 120 & 108.82 & 976 & 103.52 & 115.85 \\ 
50 &  & 132.82 & 120 & 132.82 & 134 & 136.46 & 120 & 136.87 & 1144 & 128.95 & 144.97 \\ 
100 &  & 287.46 & 73 & 265.73 & 223 & 273.37 & 120 & 276.54 & 1800 & 249.92 & 288.97 \\ \hline
\multicolumn{12}{c}{ \hotel } \\ \hline
100 &  & 502.2 & 2 & 502.2 & 2 & 502.2 & 4 & 502.2 & 7 & 502.2 & 502.8 \\ 
500 &  & 2502.2 & 8 & 2502.2 & 8 & 2502.2 & 12 & 2502.2 & 32 & 2502.2 & 2502.8 \\ 
1000 &  & 5002.2 & 17 & 5002.2 & 17 & 5002.2 & 23 & 5002.2 & 70 & 5002.2 & 5002.8 \\ 
2000 &  & 10002.2 & 35 & 10002.2 & 36 & 10002.2 & 54 & 10002.2 & 134 & 10002.2 & 10002.8 \\ \hline
\multicolumn{12}{c}{ \recycling } \\ \hline
100 & 308.79 & 308.79 & 1 & 308.79 & 1 & 314.13 & 120 & 311.28 & 490 & 308.79 & 328.37 \\ 
500 &  & 1539.56 & 19 & 1539.56 & 19 & 1577.65 & 120 & 1561.87 & 499 & 1539.56 & 1637.46 \\ 
1000 &  & 3078.02 & 40 & 3078.02 & 40 & 3157.09 & 120 & 3125.11 & 502 & 3078.02 & 3273.83 \\ 
2000 &  & 6155.07 & 120 & 6154.94 & 218 & 6315.92 & 120 & 6251.61 & 485 & 6154.94 & 6546.55 \\ \hline
\multicolumn{12}{c}{ \broadcast } \\ \hline
100 & 90.77 & 90.76 & 1 & 90.76 & 1 & 90.76 & 95 & 90.76 & 60 & 90.76 & 95.56 \\ 
500 &  & 452.74 & 16 & 452.74 & 15 & 457.19 & 61 & 456.62 & 79 & 452.74 & 476.61 \\ 
1000 &  & 905.20 & 54 & 905.20 & 65 & 915.12 & 63 & 913.92 & 76 & 905.20 & 952.93 \\ 
2000 &  & 1810.13 & 83 & 1810.13 & 85 & 1831.00 & 66 & 1828.53 & 78 & 1810.13 & 1905.56 \\ \hline
\end{tabular}
}
\caption{Upper bounds on the value of optimal policies, obtained with algorithms that find such bounds. TO and MO denote timeout (\textgreater 120s for the fast configuration) and memout (\textgreater 16GB).}  \label{tab:upperall}
\end{table}

\end{document}